\documentclass[preprint,12pt]{article}
\usepackage[top=1.1in, bottom=1.2in, left=0.8in, right=0.8in]{geometry}
\usepackage{amsmath,amsfonts,amssymb,amsthm,mathtools}
\usepackage{mathrsfs}
\usepackage{graphicx}
\usepackage{xcolor}
\usepackage{colortbl,dcolumn}
\usepackage{psfrag}
\usepackage{booktabs,multirow,array,longtable,tabularx}
\usepackage{cite}
\usepackage{url}
\usepackage{hyperref}
\hypersetup{
  colorlinks=false,
  pdfborder={0 0 0.7},
  linkbordercolor={1 0 0},
  citebordercolor={0 0.75 0},
  urlbordercolor={0 0 1}
}
\usepackage{enumitem}
\usepackage{caption}
\usepackage{subcaption}
\usepackage{float}
\usepackage[ruled,vlined,linesnumbered]{algorithm2e}
\usepackage{microtype}
\allowdisplaybreaks[4]
\numberwithin{equation}{section}

\newcommand{\R}{\mathbb{R}}
\newcommand{\E}{\mathbb{E}}

\newcommand{\TUSLA}{\mathrm{TUSLA}}
\newcommand{\dd}{\mathrm{d}}
\newcommand{\ip}[2]{\left\langle #1,#2\right\rangle}

\newtheorem{theorem}{Theorem}[section]

\newtheorem{lemma}[theorem]{Lemma}
\newtheorem{proposition}[theorem]{Proposition}
\newtheorem{corollary}[theorem]{Corollary}
\newtheorem{remark}[theorem]{Remark}
\newtheorem{example}[theorem]{Example}
\newtheorem{assumption}{Assumption}

\setlist[itemize]{leftmargin=2em}
\setlist[enumerate]{leftmargin=2em}
\DontPrintSemicolon
\SetKwInput{KwInput}{Input}
\SetKwInput{KwOutput}{Output}
\SetKw{KwReturn}{return}
\SetAlgoNlRelativeSize{-1}

\begin{document}
\title{RELTA-SGLD: Relative-Growth Localized Taming for Nonconvex Stochastic-Gradient Langevin Learning}

\author{
Yiwei Zhou\thanks{School of Mathematics and Statistics, Yunnan University, Kunming, Yunnan 650500, China. Email: \texttt{yiwei.zhou@utexas.edu}.}
\and
Ziheng Chen\thanks{School of Mathematics and Statistics, Yunnan University, Kunming, Yunnan 650500, China. Email: \texttt{12024113103@stu.ynu.edu.cn}.}
}

\date{}
\maketitle

\begin{abstract}

We introduce RELTA-SGLD, a taming scheme that stabilizes superlinear stochastic-gradient updates while reducing unnecessary suppression of the original learning drift. A threshold determines where the taming turns on, while a relative-growth principle derived from the one-step Lyapunov stability condition determines the required taming strength. Together, they produce a lighter $\lambda$-scale denominator and preserve a nonvanishing far-tail return. As a consequence, we prove polynomial moment stability and first-order stationary accuracy in both $W_1$ and $W_2$ for nonconvex SGLD with superlinearly growing stochastic-gradient oracles, improving the corresponding half-order and quarter-order bounds for comparable stochastic-gradient tamed schemes. On Fashion-MNIST under active stabilization pressure, RELTA improves the mean learning metrics over both untamed SGLD and TUSLA and remains competitive with a tuned AdamW reference. In an ordinary-training regime, its lighter localized denominator reduces unnecessary perturbation of the original update and maintains nearly untamed learning dynamics.



\end{abstract}

\section{Introduction}
\label{sec:intro}

Modern learning problems often require minimizing nonconvex objectives using stochastic-gradient evaluations.  We write the mean learning objective as
\[
u(\theta)=\mathbb E[\ell(\theta,X)],\qquad \theta\in\mathbb R^d,
\]
and consider algorithms that use an unbiased stochastic-gradient oracle
\(H(\theta,Z)\) for \(\nabla u(\theta)\).  The goal is to construct an output
\(\hat\theta\) with small expected excess risk,
\[
\mathbb E[u(\hat\theta)]-\inf_{\theta\in\mathbb R^d}u(\theta).
\]
Stochastic gradient Langevin dynamics (SGLD) is a noise-perturbed variant of
mini-batch gradient descent and has been used as a tool for nonconvex learning,
approximate optimization, and Langevin sampling
\cite{welling2011bayesian,vollmer2016exploration,raginsky2017nonconvex,xu2018global}.
Given a step size \(\lambda>0\), its explicit update is
\[
\theta_{n+1}
=
\theta_n
-
\lambda H(\theta_n,Z_{n+1})
+
\sqrt{2\lambda/\beta}\,\xi_{n+1}.
\]

The explicit nature of SGLD makes it attractive for large-scale learning, but
it also creates a stability problem when the stochastic-gradient oracle grows
faster than linearly.  In this regime, Euler-type updates may lose moment
control or even diverge numerically
\cite{hutzenthaler2011strong,sabanis2013tamed,brosse2019tamed}.  Tamed SGLD
methods address this instability by replacing the drift update with a
normalized version,
\[
\theta_{n+1}
=
\theta_n
-
\lambda
\frac{H(\theta_n,Z_{n+1})}{D_\lambda(\theta_n)}
+
\sqrt{2\lambda/\beta}\,\xi_{n+1},
\qquad
D_\lambda(\theta)\ge 1,
\]
where \(D_\lambda\) is a taming denominator, which improves
stability by shrinking large stochastic-gradient updates.

Taming, however, is not free.  It changes the dynamics of the chain and introduces bias in exchange for stability.  The main risk of instability usually comes from rare excursions to large-norm states, while most learning takes place in a much smaller region.  A denominator that is active everywhere therefore uses a global modification to control a tail risk.  It may stabilize the chain, but it may also slow ordinary learning more than necessary.

Existing tamed SGLD methods provide the starting point for this work.  TUSLA was introduced to stabilize nonconvex neural-network learning when superlinear gradients make explicit stochastic-gradient updates unreliable~\cite{lovas2023tusla}.  Subsequent developments improved this line through adaptive polygonal variants and nonsmooth-gradient extensions, including TheoPouLa, e-TH$\varepsilon$O POULA, and ReLU-network TUSLA~\cite{lim2024theopoula,lim2025etheopoula,lim2023relu}.  The TUSLA-style globally active \(\sqrt\lambda\)-scale denominator is robust and requires little detailed knowledge of the tail geometry.  At the same time, because the denominator is calibrated from absolute oracle growth, it does not distinguish the tail region where stabilization is needed from the moderate region where learning takes place.  This leaves open a more local design question: how much taming is required by the stochastic-gradient Lyapunov balance, and where should it turn on?

While our earlier work developed a local-taming viewpoint for stochastic-gradient sampling~\cite{zhou2026deterministic,zhou2026localized}, the present work specializes it to learning with SGLD.  In a one-step Lyapunov estimate, the positive quadratic oracle term \(Q\) must be controlled by the inward drift \(I\).  We therefore choose the tail strength from the growth of \(Q\) relative to \(I\), rather than from the absolute growth of the stochastic gradient.  A threshold then localizes this correction: below the threshold, the denominator is inactive and the method follows the untamed stochastic-gradient dynamics; above it, the relative-growth score supplies the needed tail control.  Together, these choices lead to the quadratic \(\lambda\)-scale denominator
\[
D_\lambda(\theta)
=
\sqrt{1+\bigl(K\lambda S_{s_0}(\theta)\bigr)^2}.
\]
We call the resulting method RELTA-SGLD, where RELTA stands for Relative-Growth Localized Taming.

Under explicit nonconvex SGLD assumptions, RELTA achieves first-order stationary accuracy simultaneously in \(W_1\) and \(W_2\).  This improves the half-order \(W_1\) and quarter-order \(W_2\) step-size accuracies available for comparable stochastic-gradient tamed schemes.  The \(W_2\) result uses the same fixed-time approximation estimates as the \(W_1\) analysis, and the SGLD target structure supplies the long-time transfer needed to turn those estimates into a stationary first-order bound.  The resulting finite-time excess-risk complexity is \(\widetilde O(\varepsilon^{-1})\), up to the fixed-temperature Gibbs bias.

Our assumptions are designed for SGLD rather than exact-gradient sampling.  The analysis uses radial tail dissipativity instead of the pairwise convexity at infinity assumed by mTULA~\cite{neufeld2022mtula}, and it requires only the averaged drift to be \(C^2\); the samplewise stochastic-gradient oracle may be nonsmooth.  Recent functional-inequality and entropy analyses of tamed Langevin samplers obtain sharper KL, total-variation, or Wasserstein consequences under exact-drift or related sampling formulations~\cite{lytrasmertikopoulos2025weaker,lytrassabanis2025isoperimetry,lytras2025ktula,ju2025modified,lytrasntousis2026tamed}.  These results are complementary to our setting: under explicit nonconvex SGLD assumptions, RELTA treats superlinear noisy oracles and gives simultaneous first-order stationary bounds in both \(W_1\) and \(W_2\).

The experiments test the learning mechanism in the full-network setting that originally motivated TUSLA.  On Fashion-MNIST under active stabilization pressure, RELTA improves the mean learning metrics over both untamed SGLD and TUSLA and remains competitive with a tuned AdamW reference.  In an ordinary-training regime, its lighter localized denominator minimally perturbs the original update and maintains nearly untamed learning dynamics.  A stationary sampling experiment in a quartic potential confirms the predicted first-order \(W_1/W_2\) behavior of RELTA and its separation from the \(\sqrt\lambda\)-scale baseline.

\section{Model and the RELTA denominator design}
\label{sec:model}

\subsection{Stochastic-gradient model and one-step balance}

Let $(Z_n)$ be i.i.d., let $(\xi_n)$ be i.i.d. standard Gaussian vectors in $\R^d$, and assume that the two sequences are independent.  We study the explicit update
\begin{equation}\label{eq:general-chain}
\theta_{n+1}
=
\theta_n
-
\lambda\frac{H(\theta_n,Z_{n+1})}{D_\lambda(\theta_n)}
+
\sqrt{2\lambda\beta^{-1}}\,\xi_{n+1},
\end{equation}
where $\lambda>0$ is the step size, $\beta>0$ is the inverse temperature, $H(\theta,Z)$ is a stochastic-gradient oracle, and $D_\lambda(\theta)\ge1$ is deterministic conditional on the current state.  We assume that the oracle is unbiased and write
\[
h(\theta)=\E[H(\theta,Z)],
\qquad
I(\theta)=\ip{\theta}{h(\theta)},
\qquad
Q(\theta)=\E\|H(\theta,Z)\|^2.
\]
Here $I$ measures the available inward mean drift, while $Q$ is the conditional quadratic size of the stochastic-gradient update.  Let $R_\lambda$ denote the Markov kernel of \eqref{eq:general-chain}; explicitly, for every measurable $f$ for which the expectation is finite,
\[
(R_\lambda f)(\theta)
=
\mathbb E\!\left[
f(\Theta_{n+1})
\,\middle|\,
\Theta_n=\theta
\right],
\]
where
\[
\Theta_{n+1}
=
\theta
-\lambda\frac{H(\theta,Z_{n+1})}{D_\lambda(\theta)}
+\sqrt{2\lambda\beta^{-1}}\,\xi_{n+1}.
\]

For $V_2(\theta)=1+\|\theta\|^2$, direct conditioning gives
\begin{equation}\label{eq:V2-exact}
R_\lambda V_2(\theta)-V_2(\theta)
=
-2\lambda\frac{I(\theta)}{D_\lambda(\theta)}
+
\lambda^2\frac{Q(\theta)}{D_\lambda(\theta)^2}
+
2\lambda\beta^{-1}d.
\end{equation}
If \(I(\theta)\) is positive outside a compact set, then the first term in \eqref{eq:V2-exact} provides the negative inward contribution needed for a Foster--Lyapunov argument.  Stability requires this term to absorb the other two positive contributions---the quadratic Euler term and the Gaussian term---outside this compact set.  The Gaussian contribution is only an \(O(\lambda)\) state-independent term, so it imposes no additional growth-order requirement on the denominator and is naturally absorbed by the residual inward drift.

Therefore, the central design question is how to choose the denominator so that the inward term can absorb the quadratic Euler term reasonably.

\subsection{TUSLA's stability mechanism and its costs}

To make the standard mechanism explicit, consider a stochastic-gradient
oracle with radial-leading structure
\begin{equation}\label{eq:radial-leading-preview}
H(\theta,u)
=
\mathsf G(\theta,u)
+
\eta\theta\|\theta\|^{2r},
\end{equation}
where $\eta>0$ and $r>0$, and
\begin{equation}\label{eq:lower-order-preview}
\|\mathsf G(\theta,u)\|
\le
\mathcal K(u)(1+\|\theta\|^q),
\qquad
q<2r+1,
\end{equation}
and $\E\mathcal K(Z)^2<\infty$.  The high-order radial term supplies
exterior dissipativity.

\begin{lemma}[Exterior growth of $I$ and $Q$]\label{lem:exterior-radial-balance}
Under \eqref{eq:radial-leading-preview}--\eqref{eq:lower-order-preview},
there exist $R<\infty$ and constants $c_I,C_Q>0$ such that, whenever
$\|\theta\|\ge R$,
\begin{equation}\label{eq:radial-exterior-IQ}
I(\theta)
\ge
c_I\|\theta\|^{2r+2},
\qquad
Q(\theta)
\le
C_Q\|\theta\|^{4r+2}.
\end{equation}
\end{lemma}

\begin{proof}
By \eqref{eq:radial-leading-preview} and Cauchy--Schwarz,
\[
\begin{aligned}
I(\theta)
&=
\eta\|\theta\|^{2r+2}
+
\E\ip{\theta}{\mathsf G(\theta,Z)}
\\
&\ge
\eta\|\theta\|^{2r+2}
-
\E\mathcal K(Z)
\bigl(\|\theta\|+\|\theta\|^{q+1}\bigr).
\end{aligned}
\]
Since $q+1<2r+2$, the radial term dominates outside a sufficiently
large ball.  Moreover,
\[
Q(\theta)
\le
2\E\|\mathsf G(\theta,Z)\|^2
+
2\eta^2\|\theta\|^{4r+2}
\le
C\bigl(1+\|\theta\|^{2q}+\|\theta\|^{4r+2}\bigr).
\]
Since $2q<4r+2$, the second bound follows after increasing $R$ and
$C_Q$ if necessary.
\end{proof}

Thus the numerator-side regularization identifies a compact set beyond which \(I(\theta)\) is positive and grows at least as
\(\|\theta\|^{2r+2}\).  The next question is how TUSLA uses the denominator \(D_\lambda\) to make this negative inward contribution absorb the quadratic Euler remainder.

The full stochastic gradient grows at order $2r+1$.  The TUSLA
denominator
\begin{equation}\label{eq:tusla-denominator}
D_{\TUSLA,\lambda}(\theta)
=
1+\sqrt\lambda\,\|\theta\|^{2r}
\end{equation}
reduces the tamed drift to at most linear growth, with a prefactor of
order $\lambda^{-1/2}$.  Consequently, the positive quadratic term in
\eqref{eq:V2-exact} is reduced to the standard $O(\lambda)$ Lyapunov
scale and can be absorbed by the negative term outside a compact set.

Because the quadratic term carries \(D_\lambda^{-2}\) while the inward term carries only \(D_\lambda^{-1}\), this denominator stabilizes the one-step Lyapunov balance.  Its limitation is that it controls the quadratic term largely in isolation and does not exploit the inward drift available at the same state.

This reveals two distinct sources of over-taming.  On the compact
region, the Foster--Lyapunov argument imposes no tail-stability
requirement, so any taming there only perturbs the original transition.
Outside this region, taming is necessary, but
Lemma~\ref{lem:exterior-radial-balance} gives
\[
I(\theta)\gtrsim \|\theta\|^{2r+2},
\qquad
Q(\theta)\lesssim \|\theta\|^{4r+2}.
\]
Balancing the two state-dependent terms in
\eqref{eq:V2-exact} can already be achieved with a
\(\lambda\|\theta\|^{2r}\)-scale denominator, whereas
\eqref{eq:tusla-denominator} uses the larger
\(\sqrt{\lambda}\|\theta\|^{2r}\) scale.

This \(\sqrt{\lambda}\) scale has a direct dynamical consequence.  In the far
tail,
\[
H(\theta,u)\sim\eta\theta\|\theta\|^{2r},
\qquad
D_{\mathrm{TUSLA},\lambda}(\theta)\sim\sqrt{\lambda}\|\theta\|^{2r},
\qquad
\lambda\frac{H(\theta,u)}{D_{\mathrm{TUSLA},\lambda}(\theta)}
\sim\eta\sqrt{\lambda}\,\theta.
\]
Thus, in the far tail, the superlinear radial drift is converted into a return whose effective coefficient is only of order \(\sqrt{\lambda}\). As \(\lambda\to0\), this coefficient vanishes, so the tamed update no longer retains the full restoring scale of the original drift. At the same time, on bounded sets, the same denominator creates an \(O(\sqrt{\lambda})\) relative perturbation even though no stabilization is needed there. This leading taming error enters the stationary Wasserstein analysis directly and limits the resulting step-size orders.

Each of these limitations has a direct learning interpretation.  First, an overly large denominator suppresses the stochastic-gradient update on ordinary training states and can slow finite-time learning. Second, as \(\lambda\downarrow0\), the vanishing far-tail return slows the recovery from large stochastic-gradient excursions. Third, weak stationary Wasserstein orders imply a larger long-run discretization bias at practical step sizes.  

\subsection{Relative-growth design}

These limitations suggest that the denominator should be designed
from the balance between the positive Euler remainder and the
available inward drift. Up to a fixed absorption constant, the required tail balance is
\begin{equation}\label{eq:relative-requirement}
\frac{\lambda^2Q(\theta)}{D_\lambda(\theta)^2}
\lesssim
\frac{\lambda I(\theta)}{D_\lambda(\theta)},
\qquad\text{equivalently}\qquad
\lambda\frac{Q(\theta)}{D_\lambda(\theta)}
\lesssim
I(\theta).
\end{equation}
Thus the denominator need not follow the absolute growth of $H$.  It needs only to compensate for the growth of $Q$ relative to $I$.  This is the $Q/I$ design principle.

Lemma~\ref{lem:exterior-radial-balance} identifies the relative-growth scale directly.  Outside a sufficiently large ball,
\[
Q(\theta)
\le
\frac{C_Q}{c_I}I(\theta)\|\theta\|^{2r}.
\]
Thus the quadratic Euler term grows relative to the available inward drift at order $\|\theta\|^{2r}$.  We therefore use the threshold score
\begin{equation}\label{eq:threshold-score}
S_{s_0}(\theta)
=
[\|\theta\|^{2r}-s_0]_+,
\qquad s_0\ge0.
\end{equation}
The threshold level \(s_0\) keeps taming inactive in the moderate region and activates it only in the far tail.  In practice, we calibrate it from a train-only untamed pilot run.  Let \(R_0\) be an empirical quantile of the recorded parameter-norm snapshots and set \(s_0=R_0^{2r}\).  In both Fashion-MNIST experiments below, \(R_0\) is chosen as the empirical \(0.8\)-quantile of the pilot norm snapshots.  For notational simplicity, we write \(S(\theta):=S_{s_0}(\theta)\) whenever the threshold \(s_0\) is fixed.

The preceding lemma gives, outside a sufficiently large ball, for another constant
\(C_{QI}>0\),
\[
Q(\theta)
\le
C_{QI} I(\theta)S_{s_0}(\theta).
\]
Substituting this relation into the one-step Lyapunov balance yields
\[
-2\lambda\frac{I(\theta)}{D_\lambda(\theta)}
+
\lambda^2\frac{Q(\theta)}{D_\lambda(\theta)^2}
\le
-\left(
2-
C_{QI}\lambda
\frac{S_{s_0}(\theta)}{D_\lambda(\theta)}
\right)
\lambda
\frac{I(\theta)}{D_\lambda(\theta)}.
\]
The displayed inequality shows that it suffices, at the level of the
one-step Lyapunov balance, to keep
\(\lambda S_{s_0}(\theta)/D_\lambda(\theta)\) uniformly below the
absorption threshold.  Equivalently, in the active tail it suffices to
choose a denominator with scale
\[
D_\lambda(\theta)
\gtrsim
\lambda S_{s_0}(\theta).
\]
This conclusion follows from an upper bound on \(Q\), and should be read
as a sufficient balance condition rather than a universal necessity
statement for arbitrary stochastic-gradient oracles.  For the radial
polynomial-growth models that motivate the construction, however, the
same scale is sharp.  Indeed, if
\[
H(\theta,Z)=\eta\theta\|\theta\|^{2r}+\zeta(\theta,Z),
\qquad
\E[\zeta(\theta,Z)\mid \theta]=0,
\]
and
\begin{equation*}
  \E[\|\zeta(\theta,Z)\|^2\mid \theta]
  =o(\|\theta\|^{4r+2}),
\end{equation*}
then
\[
I(\theta)
=\langle \theta,\E H(\theta,Z)\rangle
=\eta\|\theta\|^{2r+2},
\qquad
Q(\theta)
=\eta^2\|\theta\|^{4r+2}(1+o(1)),
\]
as \(\|\theta\|\to\infty\).  Since
\(S_{s_0}(\theta)\asymp \|\theta\|^{2r}\) in the active tail, this gives
\[
Q(\theta)/I(\theta)\asymp S_{s_0}(\theta).
\]
Thus \(\lambda S_{s_0}\) is the matching tail scale in this radial class.
We therefore set the target tail scale to \(K\lambda S_{s_0}(\theta)\),
where \(K\) controls the coefficient needed for absorption.  The
quadratic denominator introduced below satisfies
\[
D_\lambda(\theta)\ge K\lambda S_{s_0}(\theta),
\qquad
\lambda\frac{S_{s_0}(\theta)}{D_\lambda(\theta)}\le\frac{1}{K},
\qquad
2-C_{QI}\lambda\frac{S_{s_0}(\theta)}{D_\lambda(\theta)}
\ge 2-\frac{C_{QI}}{K}.
\]
Thus $S_{s_0}$ determines the active-tail stabilizing scale, while $K$ is chosen sufficiently large to keep \(2-\frac{C_{QI}}{K}\) positive.

We connect the untamed regime $D_\lambda\approx1$ to the tail-tamed regime
$D_\lambda\approx K\lambda S_{s_0}$ through the fixed quadratic denominator
\begin{equation}\label{eq:quadratic-denom}
D_\lambda(\theta)
=
\sqrt{1+\bigl(K\lambda S_{s_0}(\theta)\bigr)^2}.
\end{equation}
We refer to this as the relative-growth denominator and to the corresponding stochastic-gradient scheme as RELTA-SGLD.  Here RELTA is short for Relative-Growth Localized Taming.  The name separates the algorithmic structure from its design principle: the denominator is calibrated by the exterior $Q/I$ relative-growth law and is localized by the threshold.  The square-root form smoothly interpolates between the inactive and far-tail regimes while retaining the required far-tail scale:
\[
D_\lambda(\theta)
\sim
K\lambda S_{s_0}(\theta),
\qquad
\|\theta\|\to\infty.
\]
At the same time, it remains substantially lighter before the tail regime is reached.
This design yields two immediate improvements over the conventional
\(\sqrt{\lambda}\)-scale denominator.

First, it substantially reduces the attenuation of the learning update.
For the conventional TUSLA and relative-growth denominators,
respectively,
\[
D_\lambda^{\mathrm{T}}(\theta)=1+\sqrt{\lambda}\,\|\theta\|^{2r},
\qquad
D_\lambda^{\mathrm{QI}}(\theta)
=\sqrt{1+\bigl(K\lambda S_{s_0}(\theta)\bigr)^2},
\]
the corresponding relative attenuations satisfy
\[
0\le
1-\frac{1}{D_\lambda^{\mathrm{T}}(\theta)}
=
\frac{\sqrt{\lambda}\,\|\theta\|^{2r}}
     {1+\sqrt{\lambda}\,\|\theta\|^{2r}}
\le
\sqrt{\lambda}\,\|\theta\|^{2r},
\qquad
0\le
1-\frac{1}{D_\lambda^{\mathrm{QI}}(\theta)}
\le
\frac{K^2}{2}\lambda^2S_{s_0}(\theta)^2.
\]
Thus, before the far-tail regime is reached, the proposed denominator
replaces the leading \(O(\sqrt{\lambda})\) attenuation by a
second-order \(O(\lambda^2)\) relative attenuation. Moreover, because
\(S_{s_0}(\theta)=0\) in the untamed region, it leaves the learning
update completely unchanged there.

Second, the proposed scaling prevents the one-step far-tail return from
vanishing as \(\lambda\downarrow0\). Let
\[
H_{\mathrm{rad}}(\theta)=\eta\,\theta\|\theta\|^{2r},
\qquad
S_{s_0}(\theta)\sim\|\theta\|^{2r}
\quad (\|\theta\|\to\infty),
\]
where \(H_{\mathrm{rad}}\) denotes the leading radial regularization term.
The two one-step radial displacements satisfy
\[
-\lambda\frac{H_{\mathrm{rad}}(\theta)}
              {D_\lambda^{\mathrm{T}}(\theta)}
\sim
-\eta\sqrt{\lambda}\,\theta,
\qquad
-\lambda\frac{H_{\mathrm{rad}}(\theta)}
              {D_\lambda^{\mathrm{QI}}(\theta)}
\sim
-\frac{\eta}{K}\,\theta,
\qquad
\|\theta\|\to\infty.
\]
Hence the conventional far-tail return per step degenerates as the
stepsize decreases, whereas the \(Q/I\)-calibrated return remains of
linear order in \(\|\theta\|\), with a coefficient independent of
\(\lambda\).
This design also improves the weak stationary Wasserstein orders relative to the conventional
\(\sqrt{\lambda}\)-scale denominator. We show this in Section~\ref{sec:wasserstein}.

\section{Relative-growth stability theory}
\label{sec:stability}

This section provides the stability inputs for the stationary Wasserstein analysis.  We use the standard Lyapunov--minorization route for ergodicity of SDE approximations~\cite{mattingly2002ergodicity}.  We first derive a quadratic Foster--Lyapunov bound from the exterior \(Q/I\) balance and then extend the argument to all fixed even moments.  These estimates yield uniform moment bounds and a unique invariant law satisfying the corresponding stationary moment bounds.

\subsection{Quadratic drift from the exterior \texorpdfstring{$Q/I$}{Q/I} balance}

The quadratic case isolates the stability mechanism behind the denominator:
the inward term must absorb the quadratic Euler remainder in the one-step
Lyapunov balance.

\begin{theorem}[Quadratic stability under relative growth]\label{thm:quadratic-relative-growth}
Under \eqref{eq:radial-leading-preview}--\eqref{eq:lower-order-preview}, let
\[
S_{s_0}(\theta)=[\|\theta\|^{2r}-s_0]_+,
\qquad
D_\lambda(\theta)
=
\sqrt{1+\bigl(K\lambda S_{s_0}(\theta)\bigr)^2},
\qquad
V_2(\theta)=1+\|\theta\|^2.
\]
Then there exists a relative-growth constant $C_{\rm rel}>0$ such that, whenever $K>C_{\rm rel}/2$, one can choose
$\lambda_0\in(0,1]$ for which
\begin{equation}\label{eq:global-V2-drift-main}
R_\lambda V_2(\theta)
\le
(1-c\lambda)V_2(\theta)+B\lambda,
\qquad
0<\lambda\le\lambda_0,
\end{equation}
for constants $c>0$ and $B<\infty$ independent of $\lambda$.
\end{theorem}

\begin{proof}
Outside a sufficiently large ball,
\[
1+S_{s_0}(\theta)\asymp \|\theta\|^{2r}.
\]
Hence Lemma~\ref{lem:exterior-radial-balance} gives, for suitable constants
$C_{\rm rel},c_0>0$,
\begin{equation}\label{eq:radial-relative-growth}
Q(\theta)
\le
C_{\rm rel} I(\theta)\bigl(1+S_{s_0}(\theta)\bigr),
\qquad
\frac{I(\theta)}{1+S_{s_0}(\theta)}
\ge
c_0\|\theta\|^2.
\end{equation}
Using \eqref{eq:V2-exact}, \eqref{eq:radial-relative-growth}, and
\[
\frac{\lambda S_{s_0}}{D_\lambda}\le \frac1K,
\qquad
D_\lambda\ge1,
\]
we obtain outside that ball
\[
\begin{aligned}
R_\lambda V_2-V_2
&\le
-2\lambda\frac{I}{D_\lambda}
+C_{\rm rel}\lambda^2\frac{I(1+S_{s_0})}{D_\lambda^2}
+2\lambda\beta^{-1}d
\\
&\le
-\left(2-\frac{C_{\rm rel}}{K}-C_{\rm rel}\lambda_0\right)
\lambda\frac{I}{D_\lambda}
+2\lambda\beta^{-1}d.
\end{aligned}
\]
Choose $\lambda_0$ so that the coefficient in parentheses is positive.
Since $D_\lambda\le1+K\lambda S_{s_0}\le1+KS_{s_0}$,
\eqref{eq:radial-relative-growth} yields
\[
\frac{I}{D_\lambda}
\ge
\frac{I}{1+KS_{s_0}}
\ge
c_1\|\theta\|^2
\]
outside a sufficiently large ball.  Hence, for some $c_2>0$ and
$B_1<\infty$,
\[
R_\lambda V_2(\theta)-V_2(\theta)
\le
-c_2\lambda\|\theta\|^2+B_1\lambda
\]
throughout this exterior region.  On this large ball, the polynomial
growth assumptions imply
\[
R_\lambda V_2(\theta)-V_2(\theta)
\le
B_2\lambda
\]
uniformly for $0<\lambda\le\lambda_0$.  Combining the interior and exterior
bounds and using $V_2=1+\|\theta\|^2$ proves
\eqref{eq:global-V2-drift-main}.
\end{proof}

\subsection{Polynomial moments and uniform moment bounds}

The Wasserstein estimates in Section~\ref{sec:wasserstein} involve
polynomially weighted one-step errors and weighted transport costs.  We
therefore need moment control above order two.  The next result provides the
required Foster--Lyapunov bound for every fixed even order.

Here $K$ is the same denominator coefficient as in
\eqref{eq:quadratic-denom}.  The quadratic $Q/I$ mechanism involves the
relative-growth constant from Theorem~\ref{thm:quadratic-relative-growth},
while the radial polynomial-moment argument below requires the explicit
coefficient condition $K>\eta/2$.  Throughout, $K$ is chosen large enough
to satisfy the relevant finite collection of coefficient conditions.

For $s_0\ge0$, consider the RELTA chain
\begin{equation}\label{eq:threshold-radial-chain}
\theta_{n+1}
=
\theta_n
-
\lambda
\frac{\mathsf G(\theta_n,Z_{n+1})+\eta\theta_n\|\theta_n\|^{2r}}
{\sqrt{1+\bigl(K\lambda[\|\theta_n\|^{2r}-s_0]_+\bigr)^2}}
+
\sqrt{2\lambda\beta^{-1}}\,\xi_{n+1}.
\end{equation}

\begin{theorem}[Polynomial Foster--Lyapunov bounds]\label{thm:polynomial-foster-lyapunov}
Fix an integer $\ell\ge1$.  Suppose
\[
\|\mathsf G(\theta,u)\|
\le
\mathcal K(u)(1+\|\theta\|^q),
\qquad
q<2r+1,
\qquad
\E\mathcal K(Z_1)^{2\ell}<\infty.
\]
If
\begin{equation}\label{eq:K-threshold-main}
K>\eta/2,
\end{equation}
then there exist $\lambda_0\in(0,1]$, $a_\ell>0$, and $b_\ell<\infty$ such
that
\begin{equation}\label{eq:V2p-main}
R_\lambda V_{2\ell}(\theta)
\le
(1-a_\ell\lambda)V_{2\ell}(\theta)+b_\ell\lambda,
\qquad
V_{2\ell}(\theta)=1+\|\theta\|^{2\ell},
\end{equation}
for every $0<\lambda\le\lambda_0$.
\end{theorem}

\begin{proof}
Let \(m=2\ell\), \(t=\|\theta\|\), and write \(R_0=s_0^{1/(2r)}\) for the fixed threshold.  Then
\[
D_{\lambda,R_0}(t)
=
\sqrt{1+\bigl(K\lambda[t^{2r}-R_0^{2r}]_+\bigr)^2}.
\]
Before adding the Gaussian increment, define
\[
Y
=
\theta
-
\lambda
\frac{\mathsf G(\theta,Z)+\eta\theta t^{2r}}
{D_{\lambda,R_0}(t)}.
\]
We first prove the global bound
\begin{equation}\label{eq:pre-noise-global-main}
\E_U\|Y\|^m
\le
(1-c_m\lambda)t^m+B_m\lambda.
\end{equation}
On every fixed ball, this follows from the polynomial moment bound on
\(\mathsf G\); the details are given in
Appendix~\ref{app:high-moment-details}.  It remains to identify the negative
drift for large \(t\).

Define
\[
a_\lambda(t)
=
\frac{\lambda t^{2r}}{D_{\lambda,R_0}(t)},
\qquad
E
=
-\frac{\lambda\mathsf G(\theta,Z)}{D_{\lambda,R_0}(t)}.
\]
Thus
\begin{equation}\label{eq:Y-decomp-main}
Y
=
\bigl(1-\eta a_\lambda(t)\bigr)\theta+E.
\end{equation}
The quantity \(a_\lambda(t)\) is the effective one-step coefficient of the
leading radial term.

Since \(K>\eta/2\), choose \(A\) such that
\[
\frac1K<A<\frac2\eta.
\]
Because
\[
a_\lambda(t)
\le
\frac1K
\frac{t^{2r}}{t^{2r}-R_0^{2r}}
\qquad (t>R_0),
\]
there exists \(R\ge R_0\) such that
\[
0\le a_\lambda(t)\le A,
\qquad
t\ge R,
\qquad
0<\lambda\le1.
\]
For \(a>0\), let
\[
g_m(a)
=
\frac{1-|1-\eta a|^m}{a},
\qquad
g_m(0)=m\eta.
\]
The function \(g_m\) is continuous and strictly positive on \([0,A]\).
Hence
\[
\chi_m:=\min_{0\le a\le A}g_m(a)>0,
\]
and therefore
\begin{equation}\label{eq:radial-margin-main}
|1-\eta a_\lambda(t)|^m
\le
1-\chi_m a_\lambda(t),
\qquad
t\ge R.
\end{equation}

The lower-order stochastic-gradient displacement satisfies
\[
\|E\|
\le
a_\lambda(t)\mathcal K(Z)\delta(t)t,
\qquad
\delta(t)
=
t^{-2r-1}+t^{q-2r-1}.
\]
The condition \(q<2r+1\) gives \(\delta(t)\to0\).  Expanding the \(m\)-th
power in \eqref{eq:Y-decomp-main} and taking expectation in \(Z\) yields
\begin{equation}\label{eq:pre-noise-tail-main}
\E_U\|Y\|^m
\le
|1-\eta a_\lambda(t)|^m t^m
+
C_m a_\lambda(t)\delta(t)t^m,
\end{equation}
where \(C_m<\infty\) depends only on the fixed model parameters and the
moments of \(\mathcal K(Z)\); the expansion and the constant are recorded in
Appendix~\ref{app:high-moment-details}.  Enlarging \(R\) so that
\(C_m\delta(t)\le\chi_m/2\) for \(t\ge R\), and combining
\eqref{eq:radial-margin-main}--\eqref{eq:pre-noise-tail-main}, gives
\[
\E_U\|Y\|^m
\le
\left(1-\frac{\chi_m}{2}a_\lambda(t)\right)t^m,
\qquad
t\ge R.
\]
Moreover, this effective radial coefficient has a uniform
\(\lambda\)-order lower bound on the same exterior region.  Indeed, with
\(s=t^{2r}\) and \(s\ge R^{2r}\),
\[
D_{\lambda,R_0}(t)
=
\sqrt{1+\bigl(K\lambda[s-R_0^{2r}]_+\bigr)^2}
\le
1+K\lambda s,
\]
and therefore
\[
a_\lambda(t)
\ge
\frac{\lambda s}{1+K\lambda s}.
\]
The function \(s\mapsto \lambda s/(1+K\lambda s)\) is increasing on
\([0,\infty)\).  Hence, for \(t\ge R\) and \(0<\lambda\le\lambda_0\),
\[
a_\lambda(t)
\ge
\frac{\lambda R^{2r}}
{1+K\lambda R^{2r}}
\ge
\frac{\lambda R^{2r}}
{1+K\lambda_0 R^{2r}}
=:c_R\lambda.
\]
This proves the required negative drift outside the ball.  Combining it with
the compact-region estimate gives \eqref{eq:pre-noise-global-main}.

Finally, the centered Gaussian increment satisfies
\[
\E_\xi
\left[
\left\|Y+\sqrt{2\lambda\beta^{-1}}\,\xi\right\|^m
\,\middle|\,
Y
\right]
\le
\|Y\|^m
+
C_m^{\mathrm G}\lambda\bigl(1+\|Y\|^{m-2}\bigr).
\]
Using Young's inequality to control \(\|Y\|^{m-2}\) by an arbitrarily small
multiple of \(\|Y\|^m\) plus a constant, and then applying
\eqref{eq:pre-noise-global-main}, proves \eqref{eq:V2p-main}.  The remaining
algebra is given in Appendix~\ref{app:high-moment-details}.
\end{proof}

The condition \(K>\eta/2\) can be understood from the
large-\(\|\theta\|\) behavior of the leading radial term. In this regime, its update is approximately
\[
\theta^+
\approx
\left(1-\frac{\eta}{K}\right)\theta.
\]
Hence \(\|\theta^+\|<\|\theta\|\) if and only if \(\left|1-\frac{\eta}{K}\right|<1\) or equivalently, \(K>\eta/2\).

\begin{corollary}[Uniform moment bounds]\label{cor:uniform-moments}
Under the assumptions of Theorem~\ref{thm:polynomial-foster-lyapunov}, if
$\E V_{2\ell}(\theta_0)<\infty$, then
\[
\E V_{2\ell}(\theta_n)
\le
(1-a_\ell\lambda)^n\E V_{2\ell}(\theta_0)
+\frac{b_\ell}{a_\ell},
\qquad n\ge0.
\]
In particular, \(\sup_{n\ge0}\E V_{2\ell}(\theta_n)<\infty\).
\end{corollary}

\begin{proof}
Iterating \eqref{eq:V2p-main} and summing the resulting geometric series
gives the claim.
\end{proof}

For the stationary-accuracy analysis below, we also introduce the mean-oracle RELTA chain.  It retains the same taming denominator
and Gaussian increment as the RELTA chain, but replaces
the stochastic oracle by its conditional mean:
\begin{equation}\label{eq:mean-oracle-chain}
\bar\theta_{n+1}^{\lambda}
=
\bar\theta_n^{\lambda}
-
\lambda
\frac{h(\bar\theta_n^{\lambda})}
     {D_\lambda(\bar\theta_n^{\lambda})}
+
\sqrt{2\beta^{-1}\lambda}\,\xi_{n+1},
\qquad
h(\theta)
:=
\mathbb E\!\left[H(\theta,Z)\right].
\end{equation}
We denote its transition kernel by $\bar R_\lambda$.
This chain will be used in Section~\ref{sec:wasserstein} to
separate stochastic-gradient error from taming and
time-discretization error.  Only its finite-time moment control is
needed; no invariant-law or ergodicity result for
\eqref{eq:mean-oracle-chain} will be used.

\begin{corollary}[Finite-time moments of the mean-oracle chain]
\label{cor:mean-oracle-chain-moments}
Let $m\in\mathbb N$ and $T>0$.  Under the assumptions of
Theorem~\ref{thm:polynomial-foster-lyapunov}, there exists a constant
$C_{m,T}<\infty$, independent of
$\lambda\in(0,\lambda_0]$, such that
\begin{equation}\label{eq:mean-oracle-finite-time-moments}
\sup_{0<\lambda\le\lambda_0}
\sup_{0\le n\lambda\le T}
\mathbb E
\left[
\left\|
\bar\theta_n^\lambda
\right\|^{2m}
\right]
\le
C_{m,T}
\left(
1+
\mathbb E
\left[
\left\|
\bar\theta_0^\lambda
\right\|^{2m}
\right]
\right).
\end{equation}
Consequently, every moment of order $p\in(0,2m]$ is uniformly bounded
on the same physical-time interval.
\end{corollary}

\begin{proof}
The conditional inward-drift term and oracle-moment bound satisfy
\[
\langle\theta,h(\theta)\rangle
=
\mathbb E[\langle\theta,H(\theta,Z)\rangle\mid\theta],
\qquad
\|h(\theta)\|^{2m}
\le
\mathbb E[\|H(\theta,Z)\|^{2m}\mid\theta],
\]
where the second inequality follows from Jensen's inequality.
Hence the conditional oracle-moment bounds used in the proof of
Theorem~\ref{thm:polynomial-foster-lyapunov} remain valid after replacing
$H$ by $h$, while the centered oracle-fluctuation terms disappear.
The same Lyapunov recursion therefore yields
\eqref{eq:mean-oracle-finite-time-moments}.
\end{proof}

\subsection{Invariant law and stationary moments}

The moment bounds above also provide the stationary object used in the next section. 

\begin{corollary}[Unique invariant law and stationary moments]
\label{cor:invariant-law}
Under the assumptions of Theorem~\ref{thm:polynomial-foster-lyapunov}, for every
\(0<\lambda\le\lambda_0\), the RELTA chain \eqref{eq:threshold-radial-chain} admits a unique
invariant probability measure, denoted by \(\pi_\lambda^{\rm SG}\).
Moreover,
\begin{equation}\label{eq:stationary-moment-main}
\int V_{2\ell}(\theta)\,
\pi_\lambda^{\rm SG}(\dd\theta)
\le
\frac{b_\ell}{a_\ell}.
\end{equation}
\end{corollary}

\begin{proof}
We first verify the recurrence condition needed for existence.  By
Theorem~\ref{thm:polynomial-foster-lyapunov},
\[
R_\lambda V_{2\ell}
\le
(1-a_\ell\lambda)V_{2\ell}+b_\ell\lambda.
\]
Choose \(R<\infty\) large enough that
\(V_{2\ell}(\theta)\ge 2b_\ell/a_\ell\) whenever
\(\|\theta\|>R\).  Then
\[
R_\lambda V_{2\ell}(\theta)
\le
\left(1-\frac{a_\ell\lambda}{2}\right)
V_{2\ell}(\theta),
\qquad
\|\theta\|>R.
\]
Thus the chain has a strict expected decrease outside the ball
\(C_R=\{\|\theta\|\le R\}\).

We next show that \(C_R\) is a small set.  Since
\(\E\mathcal K(Z_1)^{2\ell}<\infty\), there exists \(M<\infty\) such that
\[
p_M:=\mathbb P\{\mathcal K(Z_1)\le M\}>0.
\]
On this event, the growth condition on \(\mathsf G\) implies that the
pre-noise update in \eqref{eq:threshold-radial-chain} remains in a bounded set,
uniformly over \(\theta\in C_R\).  Conditional on \(Z_1\), the next state is
Gaussian with covariance \(2\lambda\beta^{-1}I_d\).  Hence, on a fixed
bounded ball \(B\), these Gaussian densities have a common positive lower
bound.  Consequently, there exist \(\varepsilon_R>0\) and a probability
measure \(\nu_R\), supported on \(B\), such that
\[
R_\lambda(\theta,A)
\ge
\varepsilon_R\nu_R(A),
\qquad
\theta\in C_R,
\]
for every Borel set \(A\).

The drift condition and this small-set bound imply positive recurrence and the existence of an invariant probability measure \(\pi_\lambda^{\rm SG}\) with \(\pi_\lambda^{\rm SG}(V_{2\ell})<\infty\).  Uniqueness follows because the Gaussian increment is nondegenerate: for every \(\theta\), the one-step kernel has a density that is strictly positive on all of \(\mathbb R^d\).  The chain is therefore Lebesgue-irreducible, and has at most one invariant probability measure.

Integrating \eqref{eq:V2p-main} with respect to \(\pi_\lambda^{\rm SG}\) and using
invariance yields
\[
\pi_\lambda^{\rm SG}(V_{2\ell})
\le
(1-a_\ell\lambda)
\pi_\lambda^{\rm SG}(V_{2\ell})
+b_\ell\lambda,
\]
which proves \eqref{eq:stationary-moment-main}.
\end{proof}

\section{Wasserstein accuracy and finite-time learning risk}
\label{sec:wasserstein}

Let \(P_t\) denote the semigroup of the target diffusion
\begin{equation}\label{eq:target-diffusion-new}
\dd Z_t=b(Z_t)\dd t+\sqrt{2\beta^{-1}}\,\dd B_t,
\qquad b=-h,
\end{equation}
and let \(\pi_\beta\) denote its invariant law, whose existence and
uniqueness are established below from the Lyapunov and contraction
conditions.  In the gradient case \(h=\nabla U\), this is the Gibbs law
\(\pi_\beta(\dd x)\propto e^{-\beta U(x)}\dd x\).

Our general nonconvex result is
\begin{equation}\label{eq:main-wasserstein-rates}
W_1(\pi_\lambda^{\rm SG},\pi_\beta)\le C\lambda,
\qquad
W_2(\pi_\lambda^{\rm SG},\pi_\beta)\le C\lambda^{1/2}.
\end{equation}
In the potential case \(h=\nabla U\), the same fixed-block estimates yield
the sharper first-order bound
\[
W_2(\pi_\lambda^{\rm SG},\pi_\beta)\le C\lambda,
\]
without an additional squared-metric contraction assumption.  The proof
separates the stationary error into a stochastic-gradient channel and a
mean-discretization channel, and controls both over a fixed physical-time
block.  The resulting block estimates are then combined with the weighted
long-time contractivity of the target diffusion, or, in the potential case,
with an automatic one-sided \(W_2\) contraction toward \(\pi_\beta\).

We close the section by showing how these stationary Wasserstein
accuracy bounds transfer to finite-time excess-risk guarantees for the learning objective.

\subsection{Accuracy assumptions and the weighted metric}

The stability and moment assumptions of Section~\ref{sec:stability}
remain in force.  Throughout this section we fix a finite integer
$\ell_0$ large enough for all polynomial moment requirements used below,
depending only on the model-growth exponents such as
$(r,q,\ell_1,\ell_2,m_S,q_0)$ and on any polynomial risk-growth exponent
when excess-risk transfer is invoked; correspondingly, the radial model is
assumed to have $\E\mathcal K(Z)^{2\ell_0}<\infty$.  The following
additional conditions concern accuracy.
Write
\[
H(x,Z)=h(x)+\zeta(x,Z),
\qquad
\E[\zeta(x,Z)\mid x]=0,
\qquad
b=-h,
\]
and define the mean-shaped drift and its taming defect by
\begin{equation}\label{eq:mean-shaped-drift-new}
b_\lambda=\frac{b}{D_\lambda},
\qquad
\tau_\lambda=b_\lambda-b.
\end{equation}

\begin{assumption}[Accuracy regularity]\label{ass:w1-accuracy}
Suppose that the radial-leading decomposition can be written as
\begin{equation}\label{eq:mean-radial-decomposition}
h(x)=g(x)+\eta x\|x\|^{2r},
\qquad r>\frac12,
\end{equation}
where \(g(x)=\E[\mathsf G(x,Z)]\).  Assume:
\begin{enumerate}[label=(\roman*),leftmargin=2em]
\item There exist \(0\le \ell_1<2r\) and
\(\ell_2<\infty\) such that \(g\in C^2\) and
\begin{equation}\label{eq:lower-derivative-growth}
\|\nabla g(x)\|_{\rm op}
\le
C(1+\|x\|^{\ell_1}),
\qquad
\|\nabla^2g(x)\|_{\rm op}
\le
C(1+\|x\|^{\ell_2}).
\end{equation}
\item For some \(q_0>4\) and \(M_0<\infty\),
\begin{equation}\label{eq:centered-oracle-high-moment}
\E[\|\zeta(x,Z)\|^{q_0}\mid x]
\le C(1+\|x\|^{M_0}).
\end{equation}
\item The threshold score \(S\) is polynomially locally Lipschitz:
for some \(m_S<\infty\),
\begin{align}
S(x)&\le C(1+\|x\|^{m_S}),\label{eq:score-growth-accuracy}\\
|S(x)-S(y)|
&\le
C(1+\|x\|^{m_S-1}+\|y\|^{m_S-1})\|x-y\|.
\label{eq:score-lipschitz-accuracy}
\end{align}
\end{enumerate}
\end{assumption}

\begin{remark}[Regularity of the averaged drift]
\label{rem:mean-only-regularity}
The $C^2$ requirement in Assumption~\ref{ass:w1-accuracy}(i) is imposed only on the averaged drift
\[
g(x)=\E[\mathsf G(x,Z)].
\]
No continuity or differentiability of the sample map $x\mapsto\mathsf G(x,Z)$ is required. 
\end{remark}

The moment bounds required below follow from
Theorem~\ref{thm:polynomial-foster-lyapunov},
Corollary~\ref{cor:mean-oracle-chain-moments}, and the standard
finite-time Lyapunov estimate for the target diffusion. 

The threshold score $S(x)=[\|x\|^{2r}-s_0]_+$ satisfies \eqref{eq:score-growth-accuracy}--\eqref{eq:score-lipschitz-accuracy}.
In the radial subclass, \eqref{eq:centered-oracle-high-moment} follows by taking the
moment of \(\mathcal K(Z)\) in the radial model assumptions of Section~\ref{sec:stability}
to sufficiently high order.

\begin{lemma}[Two-point geometry of the target drift]
\label{lem:target-two-point-geometry}
Under Assumption~\ref{ass:w1-accuracy}(i), there exist
\(L_b,C_b<\infty\), \(q_b\ge0\), and \(R_b,m_b>0\) such that the
target drift \(b=-g-\eta x\|x\|^{2r}\) satisfies
\begin{align}
\ip{x-y}{b(x)-b(y)}
&\le L_b\|x-y\|^2,
\label{eq:target-global-one-sided}\\
\|b(x)-b(y)\|
&\le C_b\bigl(1+\|x\|^{q_b}+\|y\|^{q_b}\bigr)\|x-y\|,
\label{eq:target-polynomial-local-lipschitz}
\end{align}
for all \(x,y\in\R^d\). Moreover,
\begin{equation}\label{eq:target-dissipativity-infinity}
\ip{x-y}{b(x)-b(y)}
\le -m_b\|x-y\|^2
\end{equation}
whenever \(\|x\|+\|y\|\ge R_b\).
\end{lemma}

\begin{proof}
For $\Phi_r(x)=x\|x\|^{2r}$, we use the standard uniform
monotonicity estimate for the gradient of the power potential
$\|x\|^{2r+2}/(2r+2)$:
\begin{equation}
\ip{\Phi_r(x)-\Phi_r(y)}{x-y}
\ge c_r(\|x\|+\|y\|)^{2r}\|x-y\|^2.
\label{eq:radial-strong-monotonicity}
\end{equation}
Indeed, integrating the Hessian along the segment $[x,y]$ and retaining a
subsegment that stays at distance comparable to $\|x\|+\|y\|$ from the
origin gives this bound; equivalently, it is the usual uniform convexity
inequality for power functions.  The mean-value theorem and
Assumption~\ref{ass:w1-accuracy}(i) also give
\begin{equation}
|\ip{g(x)-g(y)}{x-y}|
\le C(1+\|x\|^{\ell_1}+\|y\|^{\ell_1})\|x-y\|^2.
\end{equation}
Since \(\ell_1<2r\), the radial term dominates outside a compact set,
which proves \eqref{eq:target-dissipativity-infinity}. On the compact
remainder, continuity of \(\nabla b\) gives the global bound
\eqref{eq:target-global-one-sided}. Finally, the mean-value theorem and
the polynomial bound on \(\nabla b\) give
\eqref{eq:target-polynomial-local-lipschitz}.
\end{proof}

\begin{lemma}[Taming consistency]\label{lem:taming-consistency-new}
For the quadratic denominator
\[
D_\lambda(x)=\sqrt{1+\bigl(K\lambda S(x)\bigr)^2},
\]
there are polynomial weights \(P_1,P_2\), independent of sufficiently
small \(\lambda\), such that
\begin{align}
\|\tau_\lambda(x)\|
&\le C\lambda^2P_1(x),\label{eq:taming-consistency-point}\\
\|\tau_\lambda(x)-\tau_\lambda(y)\|
&\le C\lambda^2P_2(x,y)\|x-y\|.
\label{eq:taming-consistency-two-point}
\end{align}
\end{lemma}

\begin{proof}
Set $a_\lambda=1-D_\lambda^{-1}$.  The inequality
$1-(1+z^2)^{-1/2}\le z^2/2$ gives
$0\le a_\lambda(x)\le (K^2/2)\lambda^2S(x)^2$, while the derivative of
$s\mapsto1-[1+(K\lambda s)^2]^{-1/2}$ is bounded by
$C\lambda^2(1+s)$.  Hence
\eqref{eq:score-lipschitz-accuracy}, the polynomial growth of \(b\),
and \(\tau_\lambda=-a_\lambda b\) imply
\eqref{eq:taming-consistency-point}--\eqref{eq:taming-consistency-two-point}.
\end{proof}

\begin{proposition}[Weighted contraction of the target diffusion]
\label{prop:target-weighted-contraction}
Under Assumption~\ref{ass:w1-accuracy}, the diffusion
\eqref{eq:target-diffusion-new} has a unique invariant law
\(\pi_\beta\).  Let
\[
V_2(x)=1+\|x\|^2.
\]
There exist an increasing concave function \(f\), constants
\(\epsilon,c_\rho>0\), and the product cost
\begin{equation}\label{eq:weighted-cost-new}
\rho_2(x,y)
=
f(\|x-y\|)
\{1+\epsilon V_2(x)+\epsilon V_2(y)\}.
\end{equation}
For any nonnegative lower-semicontinuous cost \(\rho\), write
\begin{equation}\label{eq:weighted-transport-definition}
\mathcal W_\rho(\mu,\nu)
:=
\inf_{\gamma\in\Gamma(\mu,\nu)}
\int_{\mathbb R^d\times\mathbb R^d}
\rho(x,y)\,\gamma(dx,dy),
\end{equation}
where \(\Gamma(\mu,\nu)\) is the set of couplings of \(\mu\) and
\(\nu\).  Then
\begin{equation}\label{eq:target-weighted-contraction-new}
\mathcal W_{\rho_2}(\mu P_t,\nu P_t)
\le
e^{-c_\rho t}\mathcal W_{\rho_2}(\mu,\nu),
\qquad t\ge0.
\end{equation}
Moreover, \(\rho_2\) satisfies a weak triangle inequality
\begin{equation}\label{eq:rho-weak-triangle-new}
\rho_2(x,z)
\le
C_\triangle\{\rho_2(x,y)+\rho_2(y,z)\},
\end{equation}
and the pointwise bridges
\begin{equation}\label{eq:rho-three-bridges}
 c\|x-y\|
\le \rho_2(x,y),
\qquad
\|x-y\|^2\le C\rho_2(x,y),
\qquad
\rho_2(x,y)\le C\Pi_2(x,y)\|x-y\|,
\end{equation}
where \(\Pi_2(x,y)=1+\|x\|^2+\|y\|^2\).
\end{proposition}

\begin{proof}
After the harmless rescaling that turns the Brownian variance in
\eqref{eq:target-diffusion-new} into the unit-noise convention of
\cite{eberleguillinzimmer2019}, Lemma~\ref{lem:target-two-point-geometry}
and the quadratic Lyapunov estimate verify the hypotheses of the
multiplicative semimetric contraction theorem
\cite[Theorem~2.2]{eberleguillinzimmer2019}.  This gives
\eqref{eq:target-weighted-contraction-new}.  The weighted-cost properties
used later in the stationary transfer argument, namely
\eqref{eq:rho-weak-triangle-new} and \eqref{eq:rho-three-bridges}, are
verified explicitly in Appendix~\ref{app:weighted-transfer-details}.
\end{proof}

The role of \(\mathcal W_{\rho_2}\) is twofold.  Unlike the ordinary Wasserstein
distances, it is contractive for the nonconvex target diffusion.
At the same time, the first two inequalities in
\eqref{eq:rho-three-bridges} transfer an
\(O(\lambda)\) bound in \(\mathcal W_{\rho_2}\) directly to
\(O(\lambda)\) in \(W_1\) and \(O(\lambda^{1/2})\) in \(W_2\).

\subsection{Stationary decomposition and fixed-block reduction}

Here $R_\lambda$ denotes the stochastic-gradient kernel of
\eqref{eq:general-chain}, $(P_t)_{t\ge0}$ is the semigroup of the target
diffusion, and $\pi_\lambda^{\rm SG}$ and $\pi_\beta$ are their invariant
laws.  For any integer $N\ge1$, invariance and the weak triangle
inequality give
\begin{align}
\mathcal W_{\rho_2}(\pi_\lambda^{\rm SG},\pi_\beta)
&\le
C_\triangle\mathcal W_{\rho_2}
\bigl(\pi_\lambda^{\rm SG}R_\lambda^N,
      \pi_\lambda^{\rm SG}P_{N\lambda}\bigr)
\notag\\
&\quad+
C_\triangle\mathcal W_{\rho_2}
\bigl(\pi_\lambda^{\rm SG}P_{N\lambda},
      \pi_\beta P_{N\lambda}\bigr).
\label{eq:stationary-first-decomposition}
\end{align}
The first term is a finite-time numerical defect, while the second is
contracted by the target diffusion.

The comparison time cannot be chosen arbitrarily.  If $N\lambda$ is too
short, the contraction factor remains close to one and the stationary
error cannot be absorbed; if the physical time is allowed to diverge,
the finite-time comparison constants need not remain uniform.  We
therefore choose $T>0$ so that
\[
q_\rho:=C_\triangle e^{-c_\rho T}<1
\]
and set
\begin{equation}\label{eq:fixed-block-definition}
n_\lambda=\left\lceil\frac{T}{\lambda}\right\rceil,
\qquad
T_\lambda=n_\lambda\lambda,
\qquad
T\le T_\lambda<T+\lambda.
\end{equation}
The two comparison estimates below are proved for an arbitrary fixed physical
horizon: throughout this subsection, if a different fixed horizon \(S>0\) is
used, \(n_\lambda(S)=\lceil S/\lambda\rceil\),
\(S_\lambda=n_\lambda(S)\lambda\), and the constant is denoted by \(C_S\).
Proposition~\ref{prop:target-weighted-contraction} then gives
\begin{equation}\label{eq:contracted-stationary-term}
C_\triangle\mathcal W_{\rho_2}
\bigl(\pi_\lambda^{\rm SG}P_{T_\lambda},
      \pi_\beta P_{T_\lambda}\bigr)
\le
q_\rho\mathcal W_{\rho_2}
\bigl(\pi_\lambda^{\rm SG},\pi_\beta\bigr).
\end{equation}
Thus only the first term in
\eqref{eq:stationary-first-decomposition} remains to be controlled.

A direct comparison between $R_\lambda^{n_\lambda}$ and
$P_{T_\lambda}$ would mix centered stochastic-gradient fluctuations
with deterministic taming and time-discretization error.  To separate
these mechanisms, insert the mean-oracle kernel $\bar R_\lambda$
introduced in \eqref{eq:mean-oracle-chain}.  It preserves the same
denominator, Gaussian increment, and Euler structure, but replaces
$H$ by $h=\E H$.  Hence
\begin{align}
\mathcal W_{\rho_2}
\bigl(\mu R_\lambda^{n_\lambda},\mu P_{T_\lambda}\bigr)
&\le
C_\triangle\mathcal W_{\rho_2}
\bigl(\mu R_\lambda^{n_\lambda},
      \mu\bar R_\lambda^{n_\lambda}\bigr)
\notag\\
&\quad+
C_\triangle\mathcal W_{\rho_2}
\bigl(\mu\bar R_\lambda^{n_\lambda},
      \mu P_{T_\lambda}\bigr).
\label{eq:block-channel-decomposition}
\end{align}
The first term is the pure stochastic-gradient channel; the second
contains only the Euler freezing and taming defects.

At this point the weighted comparison can be reduced to $W_2$.  Indeed,
for either fixed-block coupling $(X,Y)$, the last bridge in
\eqref{eq:rho-three-bridges}, Cauchy--Schwarz, and the uniform fourth
moments give
\begin{equation}\label{eq:w2-to-weighted-block-reduction}
\E\rho_2(X,Y)
\le
C\{\E\|X-Y\|^2\}^{1/2}
 \{\E \Pi_2(X,Y)^2\}^{1/2}
\le C_T\{\E\|X-Y\|^2\}^{1/2}.
\end{equation}
Consequently, it is enough to prove the two fixed-block estimates
\begin{equation}\label{eq:two-w2-block-targets}
W_2\bigl(\mu R_\lambda^{n_\lambda},
          \mu\bar R_\lambda^{n_\lambda}\bigr)\le C_T\lambda,
\qquad
W_2\bigl(\mu\bar R_\lambda^{n_\lambda},
          \mu P_{T_\lambda}\bigr)\le C_T\lambda.
\end{equation}
The next two subsections establish these estimates separately.

Because $D_\lambda(x)$ is deterministic conditional on $x$, the two
numerical kernels have the same conditional mean:
\begin{equation}\label{eq:same-conditional-mean-new}
\int z\,R_\lambda(x,\dd z)
=
\int z\,\bar R_\lambda(x,\dd z)
=
x+\lambda b_\lambda(x).
\end{equation}
This identity is the source of the extra order in the stochastic-gradient
channel.

\subsection{The centered stochastic-gradient channel}

We first estimate the stochastic-gradient component left by the
fixed-block decomposition,
\begin{equation}\label{eq:sg-channel-target-new}
W_2\!\left(
\mu R_\lambda^{n_\lambda},
\mu\bar R_\lambda^{n_\lambda}
\right)
\le C_T\lambda.
\end{equation}
Since $n_\lambda=O(\lambda^{-1})$, a squared coupling argument must
accumulate from a one-step defect of order $O(\lambda^3)$: summing
$O(\lambda^{-1})$ such defects gives a block squared error of order
$O(\lambda^2)$, and hence the required block $W_2$ error of order
$O(\lambda)$.  Equivalently, the local transition comparison must
reach the scale
\[
W_2\!\left(
R_\lambda(x,\cdot),\bar R_\lambda(x,\cdot)
\right)
=O(\lambda^{3/2})
\]
up to a polynomial state weight.  A direct synchronous coupling only
sees the displacement
$-\lambda\zeta(x,Z)/D_\lambda(x)$ and therefore gives a squared
defect of order $O(\lambda^2)$, which is insufficient over a block of
$O(\lambda^{-1})$ steps.  The missing factor of $\lambda$ must come
from a distributional effect: the Gaussian convolution in each
Langevin step, together with the conditional centering of the oracle
fluctuation.  The next lemma makes this one-step gain explicit.

\begin{lemma}[Centered Gaussian-mixture defect]
\label{lem:gaussian-smoothed-defect}
Under Assumption~\ref{ass:w1-accuracy}(ii), there exists \(M<\infty\)
such that
\begin{equation}\label{eq:gaussian-smoothed-defect-new}
W_2^2\!\left(
R_\lambda(x,\cdot),\bar R_\lambda(x,\cdot)
\right)
\le C\lambda^3(1+\|x\|^M).
\end{equation}
\end{lemma}

\begin{proof}
After translation by the common conditional mean, the comparison is
between
\[
\nu_x=\mathcal L(\sigma_\lambda G-\lambda w_x),
\qquad
\gamma_\lambda=\mathcal L(\sigma_\lambda G),
\qquad
w_x=\frac{\zeta(x,Z)}{D_\lambda(x)},
\quad
\sigma_\lambda^2=2\beta^{-1}\lambda.
\]
A synchronous coupling only sees the displacement
$\lambda w_x$ and gives a squared defect of order $\lambda^2$.
The Gaussian convolution gains one order because $w_x$ is centered.
After truncation and recentering, with an independent conditional copy
$\widetilde{w}_x^R$, the Gaussian-mixture identity is
\begin{equation}\label{eq:chi-square-mixture-identity}
1+\chi^2(\nu_x^R\Vert\gamma_\lambda)
=
\E\left[
\exp\left(
\frac{\lambda^2}{\sigma_\lambda^2}
\ip{w_x^R}{\widetilde{w}_x^R}
\right)
\middle|x
\right].
\end{equation}
Since $\lambda^2/\sigma_\lambda^2=O(\lambda)$ and the linear term
vanishes by conditional centering,
$\chi^2(\nu_x^R\Vert\gamma_\lambda)=O(\lambda^2)$ with a polynomial
state weight.  Gaussian $T_2$ then yields
\[
W_2^2(\nu_x^R,\gamma_\lambda)
\lesssim
\sigma_\lambda^2\chi^2(\nu_x^R\Vert\gamma_\lambda)
=O(\lambda^3).
\]
The truncation tail is of the same order under the conditional high-moment bound in
Assumption~\ref{ass:w1-accuracy}(ii).  The complete truncation and
mixture calculation is given in
Appendix~\ref{app:gaussian-smoothing-details}.
\end{proof}

The same-state estimate in
Lemma~\ref{lem:gaussian-smoothed-defect} is not by itself enough to
control the two chains over a block, because after the first step their
current states are no longer identical.  For states $x$ and $y$, insert
the mean-oracle kernel at $x$ and write
\begin{equation}\label{eq:sg-mean-kernel-triangle}
\begin{aligned}
W_2\!\left(R_\lambda(x,\cdot),\bar R_\lambda(y,\cdot)\right)
&\le
W_2\!\left(R_\lambda(x,\cdot),\bar R_\lambda(x,\cdot)\right)\\
&\quad+
W_2\!\left(\bar R_\lambda(x,\cdot),\bar R_\lambda(y,\cdot)\right).
\end{aligned}
\end{equation}
The first term is the centered stochastic-gradient defect already
controlled at the one-step scale \(O(\lambda^{3/2})\). For the second
term, the two mean kernels have the same Gaussian covariance, and
synchronous coupling gives
\begin{equation}\label{eq:mean-kernel-synchronous}
W_2^2\!\left(
\bar R_\lambda(x,\cdot),\bar R_\lambda(y,\cdot)
\right)
\le
\left\|
x-y+\lambda\{b_\lambda(x)-b_\lambda(y)\}
\right\|^2.
\end{equation}
The next result combines these two estimates over a fixed time block.

\begin{proposition}[First-order fixed-block stochastic-gradient defect]
\label{prop:block-sg-w2}
For every fixed \(T>0\),
\begin{equation}\label{eq:block-sg-w2-new}
W_2\!\left(
\mu R_\lambda^{n_\lambda},
\mu\bar R_\lambda^{n_\lambda}
\right)
\le C_T\lambda
\end{equation}
uniformly over initial laws in the moment class of
Theorem~\ref{thm:polynomial-foster-lyapunov}.
\end{proposition}

\begin{proof}
Start the two chains from a common initial state.  Conditional on
$X_k=x$, choose a measurable $\varepsilon$-optimal coupling $(Z^{\rm SG},\bar Z)$ of
$R_\lambda(x,\cdot)$ and $\bar R_\lambda(x,\cdot)$, independently of
$Y_k$ given $X_k$; the final estimate follows by letting $\varepsilon\downarrow0$.  Use the Gaussian coordinate of $\bar Z$ to update
the mean chain from $Y_k$, and set
\[
\Delta_k=Z^{\rm SG}-\bar Z.
\]
By \eqref{eq:same-conditional-mean-new},
$\E[\Delta_k\mid X_k,Y_k]=0$, while the coupled update satisfies
\[
X_{k+1}-Y_{k+1}
=
X_k-Y_k
+\lambda\{b_\lambda(X_k)-b_\lambda(Y_k)\}
+\Delta_k.
\]
Hence the cross term vanishes after conditioning.  With
$u_k=\E\|X_k-Y_k\|^2$, Lemma~\ref{lem:gaussian-smoothed-defect},
\eqref{eq:mean-kernel-synchronous},
\eqref{eq:target-global-one-sided},
\eqref{eq:target-polynomial-local-lipschitz}, and
Lemma~\ref{lem:taming-consistency-new} give
\begin{equation}\label{eq:sg-nonlinear-recursion}
u_{k+1}
\le
(1+C\lambda)u_k
+C\lambda^2
\E\!\left[P(X_k,Y_k)\|X_k-Y_k\|^2\right]
+C\lambda^3.
\end{equation}
By H\"older interpolation and the uniform high-moment bounds, for some
$\vartheta\in(1/2,1)$,
\[
\E\!\left[P(X_k,Y_k)\|X_k-Y_k\|^2\right]
\le C u_k^\vartheta.
\]
Since $u_0=0$, set $v_k=u_k/\lambda^2$.  Then
\[
v_{k+1}
\le
v_k+C\lambda(1+v_k+v_k^\vartheta).
\]
Using $v^\vartheta\le 1+v$ and discrete Gronwall on $k\lambda\le T$
gives $\sup_{k\le n_\lambda}v_k\le C_T$.  Therefore
$u_{n_\lambda}\le C_T\lambda^2$, which proves
\eqref{eq:block-sg-w2-new}.  The detailed weighted estimate is given in
Appendix~\ref{app:block-accumulation-details}.
\end{proof}

Proposition~\ref{prop:block-sg-w2} therefore closes the first of the
two $W_2$ terms in the fixed-block reduction.  The gain does not rely
on a vanishing-variance assumption.  At each step, the Gaussian-smoothed
kernel coupling produces a conditionally centered increment; its cross term
with the propagated state error vanishes after conditioning, so the
$O(\lambda^3)$ local squared defect accumulates to $O(\lambda^2)$ over a
fixed physical-time block.

\subsection{The mean discretization channel}

It remains to estimate the second term in the fixed-block reduction,
\begin{equation}\label{eq:mean-channel-target-new}
W_2\!\left(
\mu\bar R_\lambda^{n_\lambda},
\mu P_{T_\lambda}
\right)
\le C_T\lambda.
\end{equation}
This comparison contains two local defects.  First, the mean chain uses
the tamed drift $b_\lambda=b+\tau_\lambda$ instead of the target drift
$b$.  By Lemma~\ref{lem:taming-consistency-new},
\begin{equation}\label{eq:mean-channel-taming-scale}
\|\tau_\lambda(x)\|
\le C\lambda^2P_1(x),
\qquad
\lambda\|\tau_\lambda(x)\|
\le C\lambda^3P_1(x).
\end{equation}
Thus the one-step taming contribution is higher order on the Euler time
scale.

The leading local error is the Euler freezing remainder.  Under
additive noise, its conditional mean is $O(\lambda^2)$ and its centered
conditional variance is $O(\lambda^3)$.  These are precisely the local
scales needed to obtain a first-order $W_2$ estimate over a block of
$O(\lambda^{-1})$ steps.

\begin{lemma}[Additive-noise local remainder]
\label{lem:additive-local-remainder-new}
Let
\[
\rho_k
=
\int_{k\lambda}^{(k+1)\lambda}
\{b(Z_s)-b(Z_{k\lambda})\}\,\dd s.
\]
There is a polynomial \(P\) such that
\begin{align}
\left\|
\E[\rho_k\mid\mathcal F_{k\lambda}]
\right\|
&\le
C\lambda^2P(Z_{k\lambda}),
\label{eq:local-remainder-mean}\\
\E\!\left[
\|\rho_k-\E[\rho_k\mid\mathcal F_{k\lambda}]\|^2
\mid\mathcal F_{k\lambda}
\right]
&\le
C\lambda^3P(Z_{k\lambda}).
\label{eq:local-remainder-variance}
\end{align}
\end{lemma}

\begin{proof}
A first-order It\^o--Taylor expansion of $b(Z_s)$, together with the
polynomial derivative bounds and finite-time diffusion moments, gives
the conditional mean order $\lambda^2$ and centered variance order
$\lambda^3$.  Details are given in
Appendix~\ref{app:mean-diffusion-details}.
\end{proof}

\begin{proposition}[First-order fixed-block mean approximation]
\label{prop:mean-block-w2-new}
For every fixed \(T>0\),
\begin{equation}\label{eq:mean-block-w2-new}
W_2\!\left(
\mu\bar R_\lambda^{n_\lambda},
\mu P_{T_\lambda}
\right)
\le C_T\lambda
\end{equation}
uniformly over the same moment class.
\end{proposition}

\begin{proof}
Couple the mean-oracle chain and the target diffusion with the same
Brownian increments.  The Gaussian increments then cancel, leaving the
propagation of the existing state error, the taming term in
\eqref{eq:mean-channel-taming-scale}, and the Euler freezing remainder.
Decompose the latter into its conditional mean and centered part.
Lemma~\ref{lem:additive-local-remainder-new} gives the required
$O(\lambda^2)$ mean and $O(\lambda^3)$ centered-variance scales, while
the taming contribution is already $O(\lambda^3)$ at one step.  The
resulting squared-error recursion is
\[
w_{k+1}
\le
(1+C\lambda)w_k+C\lambda^2w_k^\vartheta+C\lambda^3,
\qquad
w_0=0.
\]
Setting $z_k=w_k/\lambda^2$ and using
$z^\vartheta\le 1+z$ reduce the recursion to
$z_{k+1}\le(1+C\lambda)z_k+C\lambda$.  Discrete Gronwall therefore
gives $w_{n_\lambda}\le C_T\lambda^2$.  See
Appendix~\ref{app:mean-diffusion-details} for the complete calculation.
\end{proof}

Combining Propositions~\ref{prop:block-sg-w2} and
\ref{prop:mean-block-w2-new} gives the full fixed-block estimate for every
fixed horizon \(S>0\): with \(n_\lambda(S)=\lceil S/\lambda\rceil\) and
\(S_\lambda=n_\lambda(S)\lambda\),
\begin{equation}\label{eq:full-block-w2-new}
W_2\!\left(
\mu R_\lambda^{n_\lambda(S)},
\mu P_{S_\lambda}
\right)
\le C_S\lambda.
\end{equation}
In particular, the estimate holds for the contraction block \(T_\lambda\)
used above and for the LSI block \(T_\lambda^\beta\) used below.

\subsection{Stationary transfer and Wasserstein consequences}

The two block propositions and
\eqref{eq:w2-to-weighted-block-reduction} yield
\begin{equation}\label{eq:weighted-blocks-new}
\mathcal W_{\rho_2}(\mu R_\lambda^{n_\lambda},
\mu\bar R_\lambda^{n_\lambda})
+
\mathcal W_{\rho_2}(\mu\bar R_\lambda^{n_\lambda},
\mu P_{T_\lambda})
\le C_T\lambda.
\end{equation}

\begin{theorem}[First-order stationary \(W_1\)]
\label{thm:sg-w1-first-order}
Let \(\pi_\lambda^{\rm SG}\) be the invariant law supplied by
Section~\ref{sec:stability}.  Under Assumption~\ref{ass:w1-accuracy},
\begin{equation}\label{eq:w1-total-first-order-new}
W_1(\pi_\lambda^{\rm SG},\pi_\beta)\le C\lambda.
\end{equation}
\end{theorem}

\begin{proof}
The target Lyapunov estimate for $V_2$ implies
$\pi_\beta(V_2)<\infty$.  Together with the polynomial moment bound for
$\pi_\lambda^{\rm SG}$ supplied by Corollary~\ref{cor:uniform-moments}
and the upper bridge in \eqref{eq:rho-three-bridges}, this gives
$\mathcal W_{\rho_2}(\pi_\lambda^{\rm SG},\pi_\beta)<\infty$.
Set $E_\lambda=\mathcal W_{\rho_2}(\pi_\lambda^{\rm SG},\pi_\beta)$.
Apply \eqref{eq:stationary-first-decomposition} with $N=n_\lambda$.
The block decomposition \eqref{eq:block-channel-decomposition} and
\eqref{eq:weighted-blocks-new} control the numerical term by
$C_T\lambda$, while \eqref{eq:contracted-stationary-term} controls the
second term by $q_\rho E_\lambda$.  Therefore
\[
E_\lambda\le C_T\lambda+q_\rho E_\lambda,
\qquad q_\rho<1.
\]
Hence \( E_\lambda
\le
\frac{C_T}{1-q_\rho}\lambda
\le C\lambda.
\)
The first bridge in \eqref{eq:rho-three-bridges} then yields
\eqref{eq:w1-total-first-order-new}.
\end{proof}

\begin{corollary}[Half-order stationary \(W_2\)]
\label{cor:unconditional-w2-half}
Under the assumptions of Theorem~\ref{thm:sg-w1-first-order},
\begin{equation}\label{eq:w2-total-half-order-new}
W_2(\pi_\lambda^{\rm SG},\pi_\beta)
\le C\lambda^{1/2}.
\end{equation}
\end{corollary}

\begin{proof}
The second bridge in \eqref{eq:rho-three-bridges} gives
$W_2^2(\mu,\nu)\le C\mathcal W_{\rho_2}(\mu,\nu)$; applying the
estimate proved in Theorem~\ref{thm:sg-w1-first-order} gives the claim.
\end{proof}

The preceding half-order \(W_2\) bound is proved before using the gradient
structure of SGLD and therefore extends to more general dissipative
stochastic-gradient Langevin-type dynamics.  We now return to the standard
SGLD setting, where the potential structure makes the required long-time
transfer automatic and upgrades the stationary \(W_2\) rate to first order.

\begin{lemma}[Automatic one-sided \(W_2\) contraction in the potential case via classical functional inequalities]
\label{lem:potential-one-sided-w2}
Assume Assumption~\ref{ass:w1-accuracy} and suppose that the mean drift is
conservative, $h=\nabla U$.  Then
$Z_\beta=\int e^{-\beta U(x)}\, \mathrm{d}x<\infty$, and
$\pi_\beta(\mathrm{d}x)=Z_\beta^{-1}e^{-\beta U(x)}\, \mathrm{d}x$ is the invariant law
of \eqref{eq:target-diffusion-new}.  Moreover, $\pi_\beta$ satisfies a
logarithmic Sobolev inequality, and there exist $T_\beta<\infty$ and
$q_\beta\in(0,1)$ such that
\begin{equation}\label{eq:potential-one-sided-w2}
W_2(\nu P_t,\pi_\beta)
\le
q_\beta W_2(\nu,\pi_\beta),
\qquad
 t\ge T_\beta,
 \quad \nu\in\mathcal P_2(\R^d).
\end{equation}
\end{lemma}

\begin{proof}
First note that the Gibbs measure is well defined.  The radial coercivity
proved from Assumption~\ref{ass:w1-accuracy}(i) gives
\[
\ip{x}{\nabla U(x)}\ge c\|x\|^{2r+2}-C,
\]
and integration along rays yields $U(x)\ge c\|x\|^{2r+2}-C$ outside a
large ball.  Hence $Z_\beta<\infty$ and $\pi_\beta$ has finite moments of
all polynomial orders.  Since the diffusion has generator
$L_\beta=\beta^{-1}\Delta-\nabla U\cdot\nabla$ and is nonexplosive by the
same Lyapunov coercivity, this Gibbs measure is invariant; by the uniqueness
established for the target diffusion, it is the law denoted by $\pi_\beta$
throughout Section~\ref{sec:wasserstein}.

The assumptions also imply the curvature lower bound needed for the
regularization step.  Since $h=\nabla U$,
\[
\nabla^2U(x)
=
\nabla g(x)
+
\eta\bigl(\|x\|^{2r}I+2r\|x\|^{2r-2}xx^\top\bigr).
\]
The radial matrix is bounded below by $\|x\|^{2r}I$, while
Assumption~\ref{ass:w1-accuracy}(i) gives
\[
\nabla^2U(x)
\succeq
\{\eta\|x\|^{2r}-C(1+\|x\|^{\ell_1})\}I.
\]
Because $\ell_1<2r$, the right-hand side is bounded below on
$\R^d$; hence $\nabla^2U\succeq-\kappa I$ for some finite $\kappa$.
Equivalently, the drift $b=-\nabla U$ satisfies the one-sided monotonicity
condition
\[
\ip{b(x)-b(y)}{x-y}\le \kappa\|x-y\|^2.
\]

We next verify the Lyapunov form of the Cattiaux--Guillin--Wu
criterion~\cite[Theorem~1.2]{cattiauxguillinwu2010}.  For $W_a(x)=e^{a\|x\|^2}$,
\[
\frac{L_\beta W_a(x)}{W_a(x)}
=
\beta^{-1}(2ad+4a^2\|x\|^2)
-2a\ip{x}{\nabla U(x)}.
\]
Choosing $a>0$ fixed and small, the superquadratic radial term implies
\[
L_\beta W_a\le -c_\beta\|x\|^2W_a+C_\beta\mathbf 1_{B_R}
\]
for a large ball $B_R$.  Together with the curvature lower bound
$\nabla^2(\beta U)\succeq-\beta\kappa I$, this Lyapunov criterion yields an
LSI for $\pi_\beta$.

Let $C_{\rm LS,\beta}$ be an LSI constant and set
$\rho_\beta=(\beta C_{\rm LS,\beta})^{-1}$ in the above generator
normalization.  Then entropy decays as
\[
\operatorname{KL}(\nu P_t\mid\pi_\beta)
\le
 e^{-2\rho_\beta t}
\operatorname{KL}(\nu\mid\pi_\beta).
\]
The same LSI implies Talagrand's $T_2$ inequality by
Otto--Villani~\cite{otto2000talagrand}.  Finally, the one-sided monotonicity
condition and nonexplosion give Wang's log-Harnack inequality for
\eqref{eq:target-diffusion-new}; see R\H{o}ckner--Wang~\cite{rocknerwang2010logharnack}.
The corresponding entropy-cost estimate follows from the pointwise
log-Harnack inequality by a standard integration argument, which we recall for
completeness.  Applying log-Harnack to the density ratio and using duality
gives
\[
\operatorname{KL}(\delta_xP_s\mid\delta_yP_s)
\le c_{\beta,\kappa}(s)\|x-y\|^2,
\qquad s>0.
\]
Integrating this inequality against an optimal coupling of $\nu$ and
$\pi_\beta$, and using joint convexity of relative entropy, yields
\begin{equation}\label{eq:entropy-cost-from-logharnack}
\operatorname{KL}(\nu P_s\mid\pi_\beta)
\le
C_{\beta,\kappa}(s)W_2^2(\nu,\pi_\beta),
\qquad s>0.
\end{equation}
This log-Harnack--entropy-cost--Talagrand assembly is classical; closely
related versions appear, for example, in Ren--Wang~\cite{renwang2021entropy}.

Combining entropy-cost regularization at time $s$, LSI entropy decay
from $s$ to $t$, and Talagrand's inequality gives
\[
W_2^2(\nu P_t,\pi_\beta)
\le
A_{\beta,s}e^{-2\rho_\beta(t-s)}W_2^2(\nu,\pi_\beta),
\qquad t\ge s,
\]
for a finite constant $A_{\beta,s}$.  Choosing $T_\beta>s$ so that the
coefficient is smaller than $q_\beta^2<1$ proves
\eqref{eq:potential-one-sided-w2}.  The constants may deteriorate with
$\beta$, but they are finite for each fixed temperature.
\end{proof}

\begin{theorem}[First-order stationary \(W_2\) in the potential case]
\label{thm:potential-first-order-w2}
Let \(\pi_\lambda^{\rm SG}\) be the invariant law supplied by
Section~\ref{sec:stability}.  Under Assumption~\ref{ass:w1-accuracy}, if
\(h=\nabla U\), then for all sufficiently small \(\lambda\),
\begin{equation}\label{eq:potential-w2-first-order}
W_2(\pi_\lambda^{\rm SG},\pi_\beta)
\le C_\beta\lambda .
\end{equation}
\end{theorem}

\begin{proof}
By Corollary~\ref{cor:uniform-moments}, the invariant law
\(\pi_\lambda^{\rm SG}\) has a finite second moment uniformly for small
\(\lambda\).  The Gibbs law \(\pi_\beta\) has finite second moment by the
coercivity verified in Lemma~\ref{lem:potential-one-sided-w2}.  Hence
\(W_2(\pi_\lambda^{\rm SG},\pi_\beta)<\infty\), so the absorption argument
below is legitimate.

Choose \(T_\beta\) and \(q_\beta<1\) from
Lemma~\ref{lem:potential-one-sided-w2}, and set
\(n_\lambda^\beta=\lceil T_\beta/\lambda\rceil\) and
\(T_\lambda^\beta=n_\lambda^\beta\lambda\).  The fixed-block estimate
\eqref{eq:full-block-w2-new}, applied with the fixed horizon \(S=T_\beta\),
gives
\[
W_2\bigl(\mu R_\lambda^{n_\lambda^\beta},\mu P_{T_\lambda^\beta}\bigr)
\le C_{\beta}\lambda
\]
uniformly over the invariant moment class.  Using invariance of
\(\pi_\lambda^{\rm SG}\) under \(R_\lambda\), invariance of \(\pi_\beta\)
under \(P_t\), and \eqref{eq:potential-one-sided-w2},
\[
\begin{aligned}
W_2(\pi_\lambda^{\rm SG},\pi_\beta)
&=W_2\bigl(\pi_\lambda^{\rm SG}R_\lambda^{n_\lambda^\beta},
          \pi_\beta P_{T_\lambda^\beta}\bigr)\\
&\le
W_2\bigl(\pi_\lambda^{\rm SG}R_\lambda^{n_\lambda^\beta},
          \pi_\lambda^{\rm SG}P_{T_\lambda^\beta}\bigr)
+W_2\bigl(\pi_\lambda^{\rm SG}P_{T_\lambda^\beta},\pi_\beta\bigr)\\
&\le C_\beta\lambda+q_\beta W_2(\pi_\lambda^{\rm SG},\pi_\beta).
\end{aligned}
\]
Absorbing the last term gives \eqref{eq:potential-w2-first-order}.
\end{proof}

\begin{corollary}[Samplewise nonsmooth stochastic gradients]
\label{cor:samplewise-nonsmooth}
The stationary Wasserstein bounds of
Theorem~\ref{thm:sg-w1-first-order} and
Corollary~\ref{cor:unconditional-w2-half}
remain valid for samplewise nonsmooth stochastic-gradient oracles,
including ReLU-type oracles, whenever the averaged drift
\(g(x)=\E[\mathsf G(x,Z)]\) satisfies
Assumption~\ref{ass:w1-accuracy}(i).  If, in addition, the averaged drift is
of potential form \(h=\nabla U\), then the first-order stationary
\(W_2\) bound of Theorem~\ref{thm:potential-first-order-w2} remains valid
under the same samplewise nonsmoothness allowance.
\end{corollary}

\begin{remark}[One-sided condition for non-gradient drifts]
\label{rem:first-order-w2}
For general non-gradient drifts, the preceding potential argument need not
apply because the invariant law is not an explicit Gibbs measure and an LSI
is not automatic.  The fixed-block estimate still yields first-order
stationary \(W_2\) under the weaker one-sided block condition
\[
W_2(\nu P_t,\pi_\beta)
\le q_2W_2(\nu,\pi_\beta),
\qquad t\in[T,T+\lambda_0],
\]
for all laws \(\nu\) in the relevant moment class and some \(q_2<1\).  The same moment bounds used in Theorem~\ref{thm:potential-first-order-w2}
ensure \(W_2(\pi_\lambda^{\rm SG},\pi_\beta)<\infty\).  Hence
\eqref{eq:full-block-w2-new} and invariance give
\[
W_2(\pi_\lambda^{\rm SG},\pi_\beta)
\le C\lambda+q_2W_2(\pi_\lambda^{\rm SG},\pi_\beta),
\]
and absorption yields \(W_2(\pi_\lambda^{\rm SG},\pi_\beta)\le C\lambda\).  The potential
case above verifies this one-sided condition automatically by functional
inequalities; no contraction of the numerical mean chain is needed.
\end{remark}

The resulting rate comparison is summarized as follows.
\begin{equation*}
\boxed{
\begin{aligned}
W_1(\pi_\lambda^{\rm SG},\pi_\beta)
&=O(\lambda),\\
W_2(\pi_\lambda^{\rm SG},\pi_\beta)
&=O(\lambda)
\quad\text{in the gradient-drift case},\\
W_2(\pi_\lambda^{\rm SG},\pi_\beta)
&=O(\lambda^{1/2})
\quad\text{in the general non-gradient drift case.}
\end{aligned}}
\end{equation*}

The first-order stationary \(W_1/W_2\) scale is already attained in a
simple fixed-batch explicit Euler stochastic-gradient Langevin setting.
The following elementary affine-Gaussian example records the matching
stationary bias in closed form.

\begin{example}[Fixed-batch affine-Gaussian Euler bias]
\label{ex:OU-first-order-bias}
Consider the one-dimensional target law
\[
\pi=N\!\left(0,(\beta a)^{-1}\right),
\qquad a>0,
\]
and the unbiased affine stochastic-gradient oracle
\[
H(x,Z)=ax+\zeta,
\qquad
\E\zeta=0,
\qquad
\operatorname{Var}(\zeta)=\frac{\sigma^2}{B}.
\]
The fixed-batch Euler SGLD update
\[
X_{n+1}
=
(1-a\lambda)X_n
-\lambda\zeta_{n+1}
+\sqrt{2\beta^{-1}\lambda}\,\xi_{n+1},
\qquad 0<\lambda<2/a,
\]
has invariant law \(\pi_\lambda=N(0,v_\lambda)\), where
\[
v_\lambda
=
\frac{2\beta^{-1}+\lambda\sigma^2/B}
{2a-a^2\lambda}.
\]
Since
\[
v_\lambda-(\beta a)^{-1}
=
\left(
\frac{\beta^{-1}}{2}
+
\frac{\sigma^2}{2aB}
\right)\lambda
+O(\lambda^2),
\qquad \lambda\downarrow0,
\]
set \(v_0=(\beta a)^{-1}\).  For centered one-dimensional Gaussian laws,
\[
W_1\!\left(N(0,v_\lambda),N(0,v_0)\right)
=
\sqrt{\frac{2}{\pi}}\,
\left|\sqrt{v_\lambda}-\sqrt{v_0}\right|,
\qquad
W_2\!\left(N(0,v_\lambda),N(0,v_0)\right)
=
\left|\sqrt{v_\lambda}-\sqrt{v_0}\right|.
\]
Moreover,
\[
\sqrt{v_\lambda}-\sqrt{v_0}
=
\frac{v_\lambda-v_0}{\sqrt{v_\lambda}+\sqrt{v_0}},
\]
and \(\sqrt{v_\lambda}+\sqrt{v_0}\to2\sqrt{v_0}>0\) as
\(\lambda\downarrow0\).  Hence
\[
W_1(\pi_\lambda,\pi)=\Theta(\lambda),
\qquad
W_2(\pi_\lambda,\pi)=\Theta(\lambda).
\]
Thus the first-order stationary scale is already attained in the
simplest affine-Gaussian fixed-batch Euler setting.
\end{example}

\subsection{Finite-time consequences for learning risk}
\label{subsec:finite-time-risk}

We now turn the stationary Wasserstein estimates into finite-time learning-risk
bounds.  The argument has two steps.  First, the block estimates used above
are iterated from arbitrary initial laws in the same moment class, giving a
finite-time Wasserstein bound with constants uniform in \(\lambda\).  Second,
in the SGLD setting, these Wasserstein bounds control the excess risk for the
Gibbs target
\(\pi_\beta(\dd x)\propto e^{-\beta U(x)}\dd x\), up to the usual
finite-temperature bias.

\subsubsection*{Finite-time Wasserstein bounds}

Let \(\mu\) be the initial law and write
\[
\mu_n:=\mu R_\lambda^n
\]
for the law of the RELTA chain after \(n\) steps.  Recall the block length \(n_\lambda=\lceil T/\lambda\rceil\) from
\eqref{eq:fixed-block-definition}, where \(T\) was chosen so that
\(q_\rho=C_\triangle e^{-c_\rho T}<1\).

\begin{lemma}[Block estimate for moment-bounded initial laws]
\label{lem:block-general-initial}
Let \(\mathcal M_M\) denote the class of probability laws \(\mu\) on
\(\R^d\) with \(\int V_{2\ell}\dd\mu\le M\), where \(\ell\le\ell_0\) is
chosen large enough for the moment requirements of
Propositions~\ref{prop:block-sg-w2} and \ref{prop:mean-block-w2-new}.  Then
there is \(C_{T,M}<\infty\), independent of \(\lambda\in(0,\lambda_0]\),
such that
\begin{equation}\label{eq:block-general-initial}
\mathcal W_{\rho_2}
\bigl(\mu R_\lambda^{n_\lambda},\mu P_{T_\lambda}\bigr)
\le
C_{T,M}\lambda,
\qquad
\mu\in\mathcal M_M.
\end{equation}
Moreover, \(\mathcal M_M\) is forward invariant for \(R_\lambda\) whenever
\(M\ge b_\ell/a_\ell\): if \(\mu\in\mathcal M_M\), then
\(\mu R_\lambda^k\in\mathcal M_M\) for all \(k\ge0\).  In particular,
\(W_2(\mu,\pi_\beta)<\infty\) for every \(\mu\in\mathcal M_M\).
\end{lemma}

\begin{proof}
Propositions~\ref{prop:block-sg-w2} and \ref{prop:mean-block-w2-new} were
proved for arbitrary initial laws in the stated moment class.  Their
constants depend on the initial law only through the finite-time moment
bounds of Corollary~\ref{cor:uniform-moments} and
Corollary~\ref{cor:mean-oracle-chain-moments}, hence only through \(M\) and
the fixed horizon \(T\).  Combining the two block estimates with
\eqref{eq:block-channel-decomposition} and
\eqref{eq:w2-to-weighted-block-reduction} gives
\eqref{eq:block-general-initial}.  The forward invariance follows from
\eqref{eq:V2p-main}: if \(\mu V_{2\ell}\le M\) and \(M\ge b_\ell/a_\ell\),
then
\[
\mu R_\lambda V_{2\ell}
\le
(1-a_\ell\lambda)M+b_\ell\lambda
\le M,
\]
and iteration gives the claim.
\end{proof}

\begin{proposition}[Finite-time estimate in the weighted Wasserstein metric]
\label{prop:finite-time-weighted}
Fix \(M\ge b_\ell/a_\ell\) and let \(\mu\in\mathcal M_M\).  Set
\(c_\star=-\log q_\rho/T_\lambda>0\).  Then for every \(k\ge0\),
\begin{equation}\label{eq:finite-time-weighted}
\mathcal W_{\rho_2}
\bigl(\mu R_\lambda^{kn_\lambda},\pi_\beta\bigr)
\le
q_\rho^{\,k}\,
\mathcal W_{\rho_2}(\mu,\pi_\beta)
+
\frac{C_{T,M}}{1-q_\rho}\,\lambda .
\end{equation}
Consequently, there are constants \(C_M<\infty\) and \(c>0\), both
independent of \(\lambda\in(0,\lambda_0]\) and of \(n\), such that
\begin{equation}\label{eq:finite-time-weighted-alln}
\mathcal W_{\rho_2}
\bigl(\mu_n,\pi_\beta\bigr)
\le
C_M\bigl(1+\mathcal W_{\rho_2}(\mu,\pi_\beta)\bigr)e^{-cn\lambda}
+
C_M\lambda,
\qquad n\ge0,
\end{equation}
where \(\mu_n:=\mu R_\lambda^n\).
\end{proposition}

\begin{proof}
Write \(\nu_k:=\mu R_\lambda^{kn_\lambda}\) for the law after \(k\) full blocks.  By the weak triangle
inequality \eqref{eq:rho-weak-triangle-new}, invariance of \(\pi_\beta\)
under \(P_{T_\lambda}\), Lemma~\ref{lem:block-general-initial} applied to
\(\nu_k\in\mathcal M_M\), and the contraction
\eqref{eq:target-weighted-contraction-new},
\[
\begin{aligned}
\mathcal W_{\rho_2}(\nu_{k+1},\pi_\beta)
&\le
C_\triangle
\mathcal W_{\rho_2}\bigl(\nu_kR_\lambda^{n_\lambda},\nu_kP_{T_\lambda}\bigr)
+
C_\triangle
\mathcal W_{\rho_2}\bigl(\nu_kP_{T_\lambda},\pi_\beta P_{T_\lambda}\bigr)
\\
&\le
C_\triangle C_{T,M}\lambda
+
q_\rho\,\mathcal W_{\rho_2}(\nu_k,\pi_\beta).
\end{aligned}
\]
Iterating this recursion and summing the geometric series gives
\eqref{eq:finite-time-weighted}, after absorbing \(C_\triangle\) into
\(C_{T,M}\).

For a general \(n\), write \(n=kn_\lambda+j\) with
\(0\le j<n_\lambda\), and set \(t_j=j\lambda\le T_\lambda\).  The arbitrary-horizon version of the fixed-block estimate applies with
physical horizon \(j\lambda\le T_\lambda\), and gives
\[
\mathcal W_{\rho_2}
\bigl(\nu_kR_\lambda^j,\nu_kP_{t_j}\bigr)
\le C_{T,M}\lambda .
\]
The finite-horizon bound for the target semigroup gives
\[
\mathcal W_{\rho_2}
\bigl(\nu_kP_{t_j},\pi_\beta P_{t_j}\bigr)
\le
C_T\mathcal W_{\rho_2}(\nu_k,\pi_\beta),
\qquad 0\le t_j\le T+\lambda_0 .
\]
Since \(\mu_n=\nu_kR_\lambda^j\) and
\(\pi_\beta P_{t_j}=\pi_\beta\), another use of the weak triangle
inequality yields
\[
\mathcal W_{\rho_2}(\mu_n,\pi_\beta)
\le
C_M\lambda
+
C_M\mathcal W_{\rho_2}(\nu_k,\pi_\beta).
\]
Moreover \(q_\rho^{\,k}\le q_\rho^{-1}e^{-c_\star n\lambda}\), because
\(kT_\lambda\ge n\lambda-T_\lambda\).  Since
\(T\le T_\lambda<T+\lambda_0\), the exponent \(c_\star\) is bounded below
by a positive constant \(c\) independent of \(\lambda\).  This yields
\eqref{eq:finite-time-weighted-alln}.
\end{proof}

\begin{corollary}[Finite-time \(W_1/W_2\) accuracy]
\label{cor:finite-time-wasserstein}
Under the assumptions of Proposition~\ref{prop:finite-time-weighted}, with
\(\theta_0\sim\mu\in\mathcal M_M\),
\begin{align}
W_1\bigl(\mu_n,\pi_\beta\bigr)
&\le
C_Me^{-cn\lambda}+C_M\lambda,
\label{eq:finite-time-w1}\\
W_2\bigl(\mu_n,\pi_\beta\bigr)
&\le
C_Me^{-cn\lambda/2}+C_M\lambda^{1/2}.
\label{eq:finite-time-w2}
\end{align}
If, in addition, either the potential-case hypotheses of
Theorem~\ref{thm:potential-first-order-w2} hold or the one-sided
squared-metric contraction condition of Remark~\ref{rem:first-order-w2}
holds, then the same iteration carried out directly in \(W_2\), using
\eqref{eq:full-block-w2-new} in place of
Lemma~\ref{lem:block-general-initial}, gives
\begin{equation}\label{eq:finite-time-w2-first-order}
W_2\bigl(\mu_n,\pi_\beta\bigr)
\le
C_Me^{-cn\lambda}+C_M\lambda .
\end{equation}
\end{corollary}

\begin{proof}
The bridges in \eqref{eq:rho-three-bridges} give
\(cW_1\le\mathcal W_{\rho_2}\) and \(W_2^2\le C\mathcal W_{\rho_2}\).
Applying these inequalities to \eqref{eq:finite-time-weighted-alln}, and
using \(\sqrt{a+b}\le\sqrt a+\sqrt b\), gives
\eqref{eq:finite-time-w1} and \eqref{eq:finite-time-w2}.  For the sharpened bound, use the block length supplied by the applicable
one-sided \(W_2\) contraction.  In the potential case this is
\(n_\lambda^\beta=\lceil T_\beta/\lambda\rceil\), and in the abstract
one-sided case it is the corresponding contraction block.  At each block end,
apply the contraction to the diffusion term and use the arbitrary-horizon
fixed-block estimate \eqref{eq:full-block-w2-new} for the numerical defect.
The forward invariance of \(\mathcal M_M\) keeps the block constants uniform
along the iteration.  The exact triangle inequality then gives the recursion
\(D_{k+1}\le C_M\lambda+qD_k\) for
\(D_k=W_2(\mu R_\lambda^{km_\lambda},\pi_\beta)\), where \(m_\lambda\) denotes the chosen contraction-block length, and the same
within-block argument gives \eqref{eq:finite-time-w2-first-order} for all
\(n\).
\end{proof}

\subsubsection*{From Wasserstein accuracy to excess risk}

In this potential setting, define
\[
R_n:=\mathbb E[U(\theta_n)]-U_\star,
\qquad
U_\star:=\inf_{x\in\R^d}U(x).
\]
When the learning objective in the introduction is denoted by \(u\) and
\(h=\nabla u\), the potential \(U\) agrees with \(u\) up to an additive
constant, so \(R_n\) is the corresponding excess risk.  Under the Gibbs-bias
assumptions below, \(U_\star\) is finite and attained.

\begin{lemma}[Risk transfer under polynomial growth]
\label{lem:finite-time-risk-transfer}
Assume \(\|\nabla U(x)\|\le C_U(1+\|x\|^{\kappa})\) for some
\(\kappa\ge0\).  Let \(\mu,\nu\) be probability laws with
\(\int\|x\|^{2\kappa}\dd\mu+\int\|x\|^{2\kappa}\dd\nu<\infty\).  Then
\begin{equation}\label{eq:risk-transfer}
\bigl|\mu U-\nu U\bigr|
\le
C_U
\left(
1+
\Bigl(\int\|x\|^{2\kappa}\mu(\dd x)\Bigr)^{1/2}
+
\Bigl(\int\|x\|^{2\kappa}\nu(\dd x)\Bigr)^{1/2}
\right)
W_2(\mu,\nu).
\end{equation}
\end{lemma}

\begin{proof}
Let \(\gamma\) be an optimal \(W_2\) coupling of \(\mu\) and \(\nu\), and
let \((X,Y)\sim\gamma\).  By the mean value theorem along the segment
\([Y,X]\) and the growth bound on \(\nabla U\),
\[
|U(X)-U(Y)|
\le
C_U
\left(1+\sup_{0\le s\le1}\|Y+s(X-Y)\|^\kappa\right)\|X-Y\|.
\]
The segment is contained in the convex hull of \(X\) and \(Y\), so
\(\|Y+s(X-Y)\|^\kappa\le C_\kappa(1+\|X\|^\kappa+\|Y\|^\kappa)\).  Hence
\[
|U(X)-U(Y)|
\le
C_U\bigl(1+\|X\|^{\kappa}+\|Y\|^{\kappa}\bigr)\|X-Y\|.
\]
Taking expectations and applying the Cauchy--Schwarz inequality,
\[
|\mu U-\nu U|
\le
C_U
\bigl\{\E(1+\|X\|^{\kappa}+\|Y\|^{\kappa})^2\bigr\}^{1/2}
\bigl\{\E\|X-Y\|^2\bigr\}^{1/2}.
\]
The second factor equals \(W_2(\mu,\nu)\).  For the first factor, use
\((a+b+c)^2\le3(a^2+b^2+c^2)\) and then take square roots; this gives
\eqref{eq:risk-transfer} after adjusting \(C_U\).
\end{proof}

The moment factor in \eqref{eq:risk-transfer} is uniformly bounded for the
two laws used below.  Choosing \(\ell\ge\kappa\) in
Theorem~\ref{thm:polynomial-foster-lyapunov} and applying
Corollary~\ref{cor:uniform-moments} gives
\begin{equation}\label{eq:uniform-risk-weight}
\sup_{n\ge0}\;
\left(\int\|x\|^{2\kappa}\mu_n(\dd x)\right)^{1/2}
\le
C_M<\infty,
\end{equation}
uniformly in \(\lambda\in(0,\lambda_0]\).  The corresponding moment of
\(\pi_\beta\) is finite by the same target Lyapunov mechanism, and in the
Gibbs setting below it is also implied by the coercivity condition in
Lemma~\ref{lem:gibbs-bias}.

\begin{lemma}[Finite-temperature Gibbs bias]
\label{lem:gibbs-bias}
Assume that \(U\) attains its minimum and satisfies the tail coercivity
condition
\[
U(x)-U_\star\ge c_{\rm tail}\|x\|^2-C_{\rm tail},
\qquad x\in\R^d,
\]
for some constants \(c_{\rm tail},C_{\rm tail}>0\).  Assume also the
local-volume condition at a minimizer: for some minimizer \(x_\star\) and
constants \(L_\star,r_\star>0\),
\[
U(x)-U_\star\le L_\star\|x-x_\star\|^2,
\qquad \|x-x_\star\|\le r_\star .
\]
Then there are constants \(C<\infty\) and \(\beta_0>0\), depending only on
these Gibbs-risk constants and on \(d\), such that
\begin{equation}\label{eq:gibbs-bias}
B_\beta:=\pi_\beta U-U_\star
\le
C\,\frac{d\log(1+\beta)}{\beta},
\qquad \beta\ge\beta_0 .
\end{equation}
\end{lemma}

\begin{proof}
We use the standard finite-temperature Gibbs-bias estimate
\cite[Proposition~11]{raginsky2017nonconvex}.  The two assumptions above
give the two inputs needed for that estimate.  The tail coercivity gives
integrability of the Gibbs density and finite polynomial moments.  The
local-volume condition gives a lower bound on the partition function: for
\(\beta\) large enough, the ball \(B(x_\star,\beta^{-1/2})\) is contained in
the local neighbourhood and
\[
\int e^{-\beta(U(x)-U_\star)}\dd x
\ge
\int_{B(x_\star,\beta^{-1/2})} e^{-\beta L_\star\|x-x_\star\|^2}\dd x
\ge
c\,\beta^{-d/2}.
\]
These are exactly the tail and local-mass inputs used in the cited
Gibbs-bias bound.  The local-volume condition is automatic, for example,
when \(\nabla U\) is locally Lipschitz near a global minimizer.
\end{proof}

\begin{theorem}[Finite-time excess risk]
\label{thm:finite-time-risk}
In the potential case described above, under the assumptions of
Corollary~\ref{cor:finite-time-wasserstein} and
Lemmas~\ref{lem:finite-time-risk-transfer} and \ref{lem:gibbs-bias}, and
for \(\beta\ge\beta_0\), there are constants \(C_M<\infty\) and \(c>0\),
independent of \(\lambda\in(0,\lambda_0]\) and \(n\), such that
\begin{equation}\label{eq:finite-time-risk}
R_n
\le
C_Me^{-cn\lambda}
+
C_M\lambda
+
C\,\frac{d\log(1+\beta)}{\beta}.
\end{equation}
\end{theorem}

\begin{proof}
Decompose
\[
R_n
=
\bigl(\mu_nU-\pi_\beta U\bigr)
+
\bigl(\pi_\beta U-U_\star\bigr).
\]
By Lemma~\ref{lem:finite-time-risk-transfer} and
\eqref{eq:uniform-risk-weight}, the first term is bounded by
\(C_MW_2(\mu_n,\pi_\beta)\).  In the potential case, the first-order
finite-time estimate \eqref{eq:finite-time-w2-first-order} applies by
Theorem~\ref{thm:potential-first-order-w2}; Lemma~\ref{lem:gibbs-bias}
controls the second term.
\end{proof}

\begin{corollary}[Iteration complexity up to temperature bias]
\label{cor:iteration-complexity}
Fix \(\beta\) and let \(\varepsilon\in(0,1)\). Let
\(N_\varepsilon\) denote the number of iterations needed to make the
finite-time numerical contribution to the excess risk at most
\(\varepsilon\), with the temperature bias \(B_\beta\) kept separate.
Under the assumptions of Theorem~\ref{thm:finite-time-risk}, there is a choice of \(\lambda\) such
that
\[
R_n\le \varepsilon+B_\beta
\]
after
\[
N_\varepsilon=\widetilde O(\varepsilon^{-1})
\]
iterations.
If \(\beta\) is chosen separately so that \(B_\beta\le\varepsilon\), then
the same bound gives total excess risk \(O(\varepsilon)\), with the usual
dependence of the constants on the chosen temperature.
\end{corollary}

\begin{proof}
Theorem~\ref{thm:finite-time-risk} gives
\[
R_n-B_\beta
\le
C_Me^{-cn\lambda}+C_M\lambda.
\]
Choose \(\lambda\asymp\varepsilon\).  The discretization term is then
\(O(\varepsilon)\).  The transient term is also \(O(\varepsilon)\) whenever
\(n\lambda\gtrsim \log(1/\varepsilon)\), hence
\[
n=\widetilde O(\varepsilon^{-1}).
\]
The temperature bias \(B_\beta\) is
not reduced by decreasing \(\lambda\) or increasing \(n\); it must either be
kept as a fixed residual or controlled by choosing \(\beta\), in which case
the constants should be read with their dependence on the chosen temperature.
\end{proof}

\section{Learning performance and mechanism verification}
\label{sec:experiments}

The experiments separate practical learning performance from asymptotic mechanism verification. We first consider two full-network Fashion-MNIST regimes. In the stabilization-active regime, untamed SGLD develops large-norm failures, so the method must provide sufficient stabilization without erasing the learning update. In the ordinary-training regime, untamed SGLD remains stable, so a successful taming rule should preserve the original learning dynamics rather than suppress them unnecessarily. These two experiments test whether RELTA adapts its intervention to the actual stability pressure. We then use a controlled one-dimensional model, where the target law is accurately available, to isolate the stationary Wasserstein rates and the error channels predicted by the theory.

The step sizes are therefore chosen for different purposes. The first experiment uses a shared larger step size, \(\lambda=0.4\), to place all methods under stabilization stress and compare the strength and learning cost of their interventions. The second uses a realistic training-scale step size to compare how TUSLA and RELTA affect ordinary learning dynamics when the untamed method is already stable. The third uses smaller step sizes to test the asymptotic stationary Wasserstein rates.

All RELTA hyperparameters were fixed before the reported evaluation and shared across seeds. The threshold was obtained from the prescribed train-only pilot procedure. No RELTA hyperparameter was tuned per seed or adjusted post hoc.

\subsection{Stabilization-active learning: sufficient control with lighter intervention}
\label{subsec:fmnist}

The first Fashion-MNIST experiment considers a practical learning regime with active stabilization pressure.  RELTA, TUSLA, and untamed SGLD are run at the common aggressive step size \(\lambda=0.4\), using the same ReLU multilayer perceptron
\[
784\!\to\!256\!\to\!128\!\to\!10,
\]
the same mini-batch size, update budget, data split, and paired randomness across \(10\) seeds.  RELTA and TUSLA share the same regularized numerator and differ only in the taming denominator; RELTA uses the quadratic denominator in~\eqref{eq:quadratic-denom}, with \(K=2\) and a train-only \(0.8\)-quantile pilot threshold.  Untamed SGLD is the unstabilized reference, and validation-tuned AdamW~\cite{loshchilov2019adamw} is included as an external optimization baseline.  We report final metrics and normalized area-under-learning-curve (nAULC) summaries as mean (SD); the supplementary reproducibility package contains the run scripts, fixed seeds, pilot-threshold protocol, nAULC convention, and per-seed output files used to generate the tables.

\begin{table}[H]
\centering
\caption{Fashion-MNIST learning performance under active stabilization pressure. Lower loss metrics and higher accuracy metrics are better.}
\label{tab:fmnist-confirmatory}
\small
\begin{tabular}{lccccc}
\toprule
method
& loss nAULC
& accuracy nAULC
& final loss
& final accuracy
& mean \(D\) \\
\midrule
RELTA
& \(0.4615\,(0.0042)\)
& \(0.8355\,(0.0016)\)
& \(0.4250\,(0.0140)\)
& \(0.8579\,(0.0049)\)
& \(2.386\,(0.020)\) \\
TUSLA
& \(0.5711\,(0.0041)\)
& \(0.7977\,(0.0016)\)
& \(0.4590\,(0.0103)\)
& \(0.8393\,(0.0043)\)
& \(9.389\,(0.011)\) \\
untamed
& \(0.7940\,(0.9128)\)
& \(0.8171\,(0.0172)\)
& \(0.4727\,(0.0401)\)
& \(0.8401\,(0.0153)\)
& \(1.000\,(0.000)\) \\
AdamW
& \(0.4703\,(0.0024)\)
& \(0.8365\,(0.0009)\)
& \(0.4273\,(0.0122)\)
& \(0.8558\,(0.0045)\)
& -- \\
\bottomrule
\end{tabular}
\end{table}

Table~\ref{tab:fmnist-confirmatory} reports mean (SD).  For untamed SGLD, the loss nAULC distribution is strongly right-tailed: its median (IQR) is \(0.497\,(0.489,0.515)\), while the mean \(0.794\) is driven by two large-norm tail runs.  These same two seeds reach final parameter norms of approximately \(115.5\) and \(170.0\), compared with roughly \(24\)--\(34\) for the other untamed runs. TUSLA prevents these tail failures, but its mean denominator of \(9.389\) reveals substantial suppression of the learning update. RELTA uses a much lighter mean denominator of \(2.386\) and outperforms TUSLA on all four learning metrics in every paired seed. Compared with untamed SGLD, RELTA improves the mean metrics by eliminating the heavy tail; the robust median confirms that nonfailure untamed trajectories can learn comparably fast, but without stabilization they are unreliable at this step size.

AdamW provides a strong external optimization reference. It has a small advantage in accuracy nAULC, whereas RELTA achieves lower loss nAULC in all \(10\) paired runs and comparable final performance.

\paragraph{Finite-time hitting times.}
As a finite-time learning diagnostic, we also report the median number of steps needed to hit several prescribed loss
thresholds, together with the number
of seeds that reached the threshold.   For a threshold \(L\), define
\[
\tau(L)=\min\{k:\mathrm{loss}_k\le L\},
\]

\begin{table}[H]
\centering
\small
\caption{Loss-threshold hitting steps on Fashion-MNIST under the active-stabilization setting.}
\label{tab:fmnist-hitting-time}
\begin{tabular}{lcccc}
\toprule
Method & \(\tau(0.8)\) & \(\tau(0.6)\) & \(\tau(0.5)\) & \(\tau(0.4)\) \\
\midrule
Untamed SGLD & \(100\) (10/10) & \(250\) (10/10) & \(500\) (10/10) & -- (0/10) \\
TUSLA        & \(275\) (10/10) & \(725\) (10/10) & \(1600\) (10/10) & -- (0/10) \\
RELTA        & \(100\) (10/10) & \(250\) (10/10) & \(475\) (10/10) & \(2400\) (4/10) \\
AdamW reference & \(100\) (10/10) & \(250\) (10/10) & \(600\) (10/10) & -- (0/10) \\
\bottomrule
\end{tabular}
\end{table}
Entries report median hitting steps over successful seeds; the number of successful seeds is shown in parentheses.

RELTA reaches the moderate thresholds as quickly as untamed SGLD and the AdamW reference, while being substantially faster than TUSLA.  At the lowest loss threshold \(0.4\), RELTA is the only method that reaches the threshold in any seed.  This supports the finite-time interpretation of the learning experiments: RELTA preserves enough drift for rapid progress while retaining the stabilization needed at an aggressive step size.

\subsection{Ordinary learning: preservation without unnecessary suppression}
\label{subsec:fmnist-horizon}

We next isolate the dynamical cost of taming in a conventional full-network Langevin experiment, using the same Fashion-MNIST MLP architecture and paired-seed protocol as above.  The methods share the empirical-risk oracle, radial regularizer, mini-batch size, and update budget; only the taming denominator differs.  RELTA again uses the quadratic denominator in~\eqref{eq:quadratic-denom}, with \(K=2\) and a train-only \(0.8\)-quantile pilot threshold.  We run at \(\lambda=2\times10^{-4}\) to physical time \(T=4.8\).

For a denominator sequence \((D_k)\), define the retained effective drift horizon by
\[
T_{\mathrm{eff}}(n)
=
\sum_{k=0}^{n-1}\frac{\lambda}{D_k}.
\]
For untamed SGLD, $T_{\mathrm{eff}}/T=1$; smaller values indicate stronger suppression of the drift update.  Table~\ref{tab:fmnist-horizon} reports this mechanism diagnostic together with the learning metrics, all summarized as mean (SD) over \(10\) paired seeds.  The supplementary reproducibility package records the precise split, normalization, checkpointing, nAULC convention, fixed seeds, and per-seed outputs.

\begin{table}[H]
\centering
\caption{Full-network Fashion-MNIST SGLD horizon experiment. Lower loss nAULC and higher accuracy and effective-horizon metrics are better.}
\label{tab:fmnist-horizon}
\small
\begin{tabular}{lccccc}
\toprule
method & loss nAULC & accuracy nAULC & \(T_{\mathrm{eff}}/T\) & final loss & final accuracy\\
\midrule
TUSLA & \(2.110\,(0.079)\) & \(0.252\,(0.029)\) & \(0.078\,(0.0001)\) & \(2.185\,(0.206)\) & \(0.310\,(0.054)\)\\
RELTA & \(1.016\,(0.026)\) & \(0.663\,(0.011)\) & \(0.905\,(0.0003)\) & \(0.740\,(0.023)\) & \(0.739\,(0.010)\)\\
untamed SGLD & \(1.001\,(0.025)\) & \(0.668\,(0.010)\) & \(1.000\,(0.0000)\) & \(0.703\,(0.021)\) & \(0.751\,(0.008)\)\\
\bottomrule
\end{tabular}
\end{table}

Untamed SGLD and RELTA have nearly identical finite-horizon learning dynamics, whereas TUSLA loses most of its effective drift horizon.  RELTA improves both validation-loss and validation-accuracy nAULC relative to TUSLA in all 10 paired seeds.  Untamed SGLD is slightly better than RELTA, but the gap is small: accuracy nAULC differs by only \(0.0048\), and RELTA retains \(90.5\%\) of the untamed effective horizon at \(T=4.8\).  TUSLA retains only \(7.8\%\) and underperforms both methods by a wide margin.

Over this finite horizon, taming acts mainly as a safeguard rather than an accelerator.  At this training-scale step size, untamed SGLD remains well behaved, and RELTA accordingly preserves nearly untamed learning dynamics while retaining an exterior stability mechanism; TUSLA instead applies the correction throughout the learning region.

\subsection{Mechanism verification: stationary accuracy and error channels}
\label{subsec:stationary-rate}

We use the one-dimensional normalized-design phase-retrieval risk
\[
\ell_n(x)
=
\frac1{4n}\sum_{i=1}^n (x^2-y_i)^2+\frac{\rho}{2}x^2,
\qquad
U(x):=\mathbb E[\ell_n(x)]
=
\frac{x^4}{4}-\frac{1.75}{2}x^2+\text{constant}.
\]

The target is nonconvex but only mildly metastable: its wells are at $\pm1.32288$, its barrier height is $0.765625$. The target density is proportional to \(e^{-U(x)}\), corresponding to \(\beta=1\).

We use a conditionally unbiased Gaussian stochastic-gradient oracle whose conditional variance matches that of a batch-one component gradient. Since \(\nabla U_i(x)=x^3+(\rho-y_i)x,\) the gradient noise grows at a lower order than the leading cubic drift and therefore does not affect the dominant tail stability mechanism, as required by the theory.

We compare the invariant laws of the mean-oracle RELTA chain,
the stochastic-gradient RELTA chain, and a
\(\sqrt{\lambda}\)-scale TUSLA-form baseline over
\[
\lambda\in\{2^{-4},2^{-5},2^{-6},2^{-7},2^{-8}\}.
\]
For each \(\lambda\), we use 10 seeds and 700 antithetic pairs per seed,
with burn-in physical time \(60\) and sampling physical time \(300\).
Wasserstein errors are computed by exact one-dimensional discrete optimal
transport against a high-accuracy quadrature approximation of the Gibbs target.

\begin{figure}[H]
    \centering
    \includegraphics[width=0.9\textwidth]{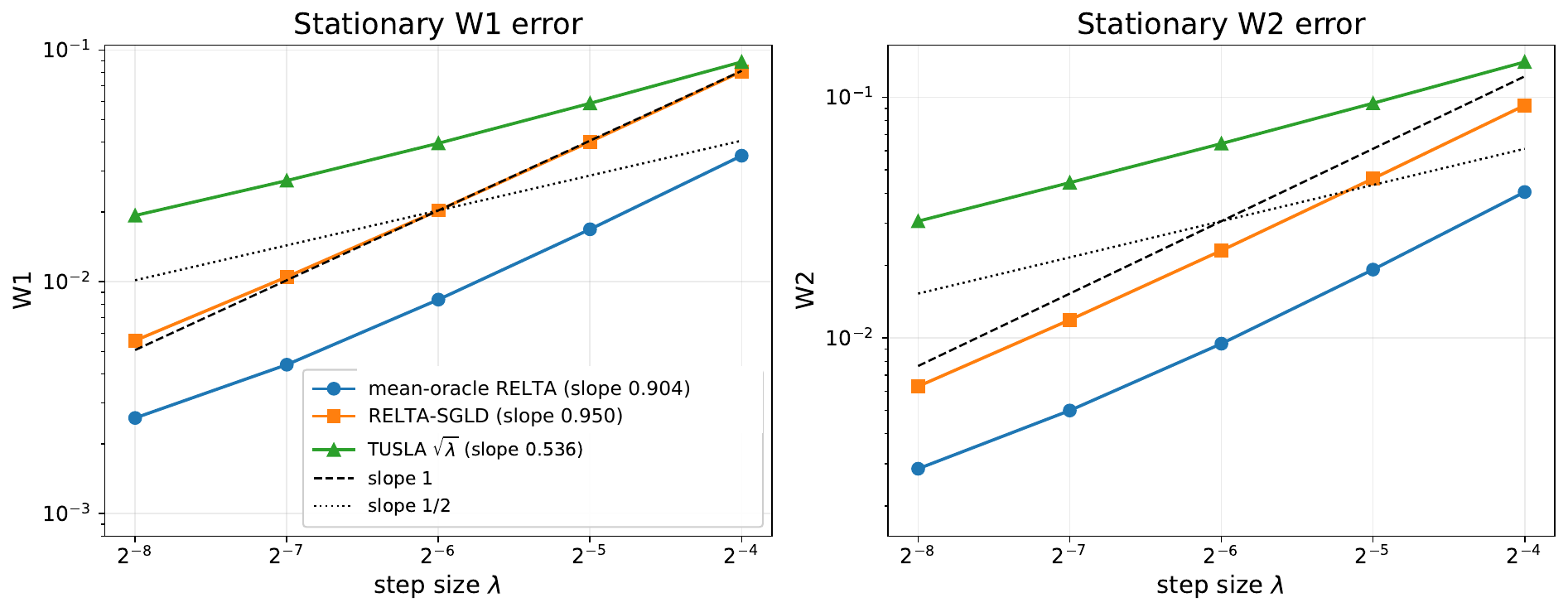}
    \caption{Stationary Wasserstein errors in the smooth quartic model.  The $\lambda$-scale chains display near-first-order $W_1$ and $W_2$ slopes, whereas the $\sqrt\lambda$ TUSLA-form baseline displays slopes close to one half.}
    \label{fig:stationary-rates}
\end{figure}

\begin{table}[H]
\centering
\caption{Pooled stationary log--log slopes from the four smallest step sizes.}
\label{tab:stationary-rate-slopes}
\small
\begin{tabular}{lrr}
\toprule
method & $W_1$ slope & $W_2$ slope\\
\midrule
mean-oracle RELTA & 0.904 & 0.917\\
RELTA & 0.950 & 0.956\\
TUSLA $\sqrt\lambda$ & 0.536 & 0.543\\
\bottomrule
\end{tabular}
\end{table}

Table~\ref{tab:stationary-rate-slopes} summarizes the stationary fits: both the mean-oracle and stochastic-gradient RELTA chains exhibit near-first-order \(W_1\) and \(W_2\) slopes, whereas the
\(\sqrt{\lambda}\)-scale TUSLA baseline remains close to half order.  The one-dimensional quartic double-well target is included as a case covered by Theorem~\ref{thm:potential-first-order-w2}.  In the normalization used in the experiment,
\[
U''(x)=3x^2-1.75\ge -1.75,
\qquad
xU'(x)=x^4-1.75x^2,
\]
so the target is globally semiconvex and coercive in the Lyapunov sense.  The near-unit \(W_2\) slopes in Table~\ref{tab:stationary-rate-slopes} are consistent with the first-order \(W_2\) rate predicted by the theorem.

\section{Discussion}
\label{sec:discussion}

Taming should respond to the source of instability, not to the absolute size of the stochastic gradient.  RELTA uses the $Q/I$ balance to set the tail strength and a threshold to localize its activation.  This yields a $\lambda$-scale denominator that approaches one faster than the usual $\sqrt\lambda$-scale design, preserves the learning region, and enables sharper stationary Wasserstein rates.

A remaining theoretical distinction concerns the scope of first-order stationary $W_2$ accuracy.  In the SGLD setting, the present assumptions automatically supply the functional-inequality inputs needed to transfer the first-order fixed-block error to invariant laws.  For general non-gradient drifts, however, an invariant-law LSI or an equivalent one-sided squared-metric contraction input is still needed. Developing verifiable criteria for this non-gradient case, and testing nonradial calibrations on larger learning problems, are natural directions beyond the present scope.

From a practical perspective, the next step is to evaluate RELTA on
larger and more heterogeneous network architectures.  In such settings, the need for stabilization may vary across parameter blocks and over the course of training.  This raises the question of whether one can replace fixed threshold and coefficient by inexpensive online estimates of the local $Q/I$ balance, so that taming can adapt to the observed stability pressure without introducing extensive tuning.

\appendix

\section{Details for the stationary Wasserstein analysis}
\label{app:wasserstein-details}

\subsection{Weighted metric verification and stationary transfer}
\label{app:weighted-transfer-details}

We verify that the weighted semimetric used in
Proposition~\ref{prop:target-weighted-contraction} satisfies the EGZ
hypotheses and the bridge estimates needed for stationary transfer.  The
target generator is
\(\mathcal L\varphi(x)=\ip{b(x)}{\nabla\varphi(x)}+\beta^{-1}\Delta\varphi(x)\),
with \(b=-h\).  Set \(V_2(x)=1+\|x\|^2\) and
\(G_2(x,y)=1+\epsilon V_2(x)+\epsilon V_2(y)\).  Radial dissipativity gives
\begin{align*}
\mathcal LV_2(x)
&=2\ip{x}{b(x)}+2d/\beta
\le C_V-c_VV_2(x),
&
\frac{\|\nabla V_2(x)\|}{V_2(x)}
&=\frac{2\|x\|}{1+\|x\|^2}
\le 1.
\end{align*}
Together with the generalized one-sided Lipschitz and exterior
contractivity estimates in Lemma~\ref{lem:target-two-point-geometry}, these
are the drift and Lyapunov hypotheses of the multiplicative semimetric
contraction theorem of
Eberle--Guillin--Zimmer~\cite[Theorem~2.2]{eberleguillinzimmer2019}, after
rescaling the Brownian variance.  Hence the theorem constructs an increasing
concave distance function \(f\) and constants \(\epsilon,c_\rho>0\) such that
\(\rho_2(x,y)=f(\|x-y\|)G_2(x,y)\) satisfies
\eqref{eq:target-weighted-contraction-new}.

We next verify the elementary cost comparisons used in the fixed-block
argument.  The EGZ construction gives a bounded increasing concave function
\(f\), with \(f(0)=0\), \(f(r)>0\) for \(r>0\), and constants
\(r_\ast\in(0,1]\), \(0<c_f\le C_f<\infty\), such that
\begin{equation}\label{eq:app-f-properties}
 c_f r\le f(r)\le C_f r\quad(0\le r\le r_\ast),
 \qquad
 f(r)\ge c_f\quad(r\ge r_\ast),
 \qquad
 f(r)\le C_f r\quad(r\ge0).
\end{equation}
Concavity and \(f(0)=0\) also imply subadditivity:
\(f(a+b)\le f(a)+f(b)\) for \(a,b\ge0\).

The linear bridge follows directly when \(r=\|x-y\|\le r_\ast\).  If
\(r>r_\ast\), then
\(r^2\le2\|x\|^2+2\|y\|^2\le C\{V_2(x)+V_2(y)\}\), and
\eqref{eq:app-f-properties} yields
\(\rho_2(x,y)\ge c_fG_2(x,y)\ge c\max\{r,r^2\}\ge cr\).  Thus
\(c\|x-y\|\le\rho_2(x,y)\).  The same split gives the quadratic bridge:
for \(r\le r_\ast\), \(r^2\le Cf(r)\le C\rho_2(x,y)\), while for
\(r>r_\ast\) the preceding tail argument gives \(r^2\le C\rho_2(x,y)\).
Finally, the upper bound in \eqref{eq:app-f-properties} and the definition
of \(G_2\) imply
\(\rho_2(x,y)\le C\Pi_2(x,y)\|x-y\|\), where
\(\Pi_2(x,y)=1+\|x\|^2+\|y\|^2\).  This proves the three bridge inequalities
in \eqref{eq:rho-three-bridges}.

It remains to check the weak triangle inequality.  Write
\(G_2(u,v)=1+\epsilon V_2(u)+\epsilon V_2(v)\) and choose
\(\delta\in(0,r_\ast]\).  If \(\|u-v\|\le\delta\), then
\begin{equation}\label{eq:app-local-weight-comparable}
G_2(u,w)\le C_\delta G_2(v,w),
\qquad w\in\mathbb R^d,
\end{equation}
because \(V_2(u)\le C_\delta V_2(v)\).  Let
\(a=\|x-y\|\) and \(b=\|y-z\|\).  If both are at most \(\delta\), then
local comparability and subadditivity give
\(\rho_2(x,z)\le C\{\rho_2(x,y)+\rho_2(y,z)\}\).  If, say,
\(a>\delta\) and \(b\le\delta\), then
\(G_2(x,z)\le CG_2(x,y)\), while boundedness of \(f\) and
\(f(a)\ge f(\delta)>0\) give \(\rho_2(x,z)\le C\rho_2(x,y)\); the case
\(b>\delta\), \(a\le\delta\) is symmetric.  If both \(a,b>\delta\), then
\(G_2(x,z)\le C\{G_2(x,y)+G_2(y,z)\}\), and the same boundedness/lower-bound
argument gives the two-term estimate.  Therefore
\[
\rho_2(x,z)
\le
C_\triangle\{\rho_2(x,y)+\rho_2(y,z)\},
\]
which is \eqref{eq:rho-weak-triangle-new}.  The same constant also gives
the corresponding weak triangle inequality at the transport-cost level used in
the main text: take two \(\varepsilon\)-optimal couplings for
\(\mathcal W_{\rho_2}(\mu,\nu)\) and \(\mathcal W_{\rho_2}(\nu,\omega)\),
glue them along the common middle marginal \(\nu\), integrate the preceding
pointwise estimate, and then let \(\varepsilon\downarrow0\).

\subsection{Gaussian smoothing of the centered oracle noise}

\label{app:gaussian-smoothing-details}

We prove Lemma~\ref{lem:gaussian-smoothed-defect}.  Let
$L_\lambda=\delta\lambda^{-1/2}$ and define the centered truncation
\[
w_x^R
=
w_x\mathbf 1_{\{\|w_x\|\le L_\lambda\}}
-
\E[w_x\mathbf 1_{\{\|w_x\|\le L_\lambda\}}\mid x].
\]
For $\nu_x^R=\mathcal L(\sigma_\lambda G-\lambda w_x^R)$, a
synchronous coupling and the conditional $q_0$-moment bound give
\begin{align}
W_2^2(\nu_x,\nu_x^R)
&\le
C\lambda^2
\E[\|w_x\|^2\mathbf 1_{\{\|w_x\|>L_\lambda\}}\mid x]
\notag\\
&\le
C\lambda^{(q_0+2)/2}(1+\|x\|^M)
\le
C\lambda^3(1+\|x\|^M).
\label{eq:truncation-transport-app}
\end{align}
Let $\widetilde{w}_x^R$ be an independent conditional copy.  Direct
integration of the Gaussian likelihood ratios gives
\eqref{eq:chi-square-mixture-identity}.  Since
$\|w_x^R\|\le2L_\lambda$, the exponent is uniformly bounded.  Its
linear term has conditional expectation zero, and therefore
\[
\chi^2(\nu_x^R\Vert\gamma_\lambda)
\le
C\lambda^2
\bigl(\E[\|w_x^R\|^2\mid x]\bigr)^2
\le
C\lambda^2(1+\|x\|^M).
\]
Using ${\rm KL}\le\log(1+\chi^2)\le\chi^2$ and Gaussian $T_2$,
\[
W_2^2(\nu_x^R,\gamma_\lambda)
\le
2\sigma_\lambda^2{\rm KL}(\nu_x^R\Vert\gamma_\lambda)
\le
C\lambda^3(1+\|x\|^M).
\]
Together with \eqref{eq:truncation-transport-app}, this proves
\eqref{eq:gaussian-smoothed-defect-new}.

\subsection{Accumulation of the stochastic-gradient block defect}
\label{app:block-accumulation-details}

For Proposition~\ref{prop:block-sg-w2}, write
\(e_k=X_k-Y_k\) and \(u_k=\E\|e_k\|^2\).  Under the conditional coupling used
in the proof, the one-step defect \(\Delta_k\) is centered and satisfies
\[
\E[\Delta_k\mid X_k,Y_k]=0,
\qquad
\E[\|\Delta_k\|^2\mid X_k]
\le C\lambda^3(1+\|X_k\|^M).
\]
Let
\[
R_k=e_k+\lambda\{b_\lambda(X_k)-b_\lambda(Y_k)\}.
\]
Then \(R_k\) is measurable with respect to the current state pair, and the
centered square expansion gives
\begin{equation}\label{eq:app-centered-square-compressed}
\E\|e_{k+1}\|^2
=
\E\|R_k\|^2+
\E\|\Delta_k\|^2.
\end{equation}
Writing \(b_\lambda=b+\tau_\lambda\), we split
\[
R_k=e_k+\lambda\{b(X_k)-b(Y_k)\}
     +\lambda\{\tau_\lambda(X_k)-\tau_\lambda(Y_k)\}.
\]
The one-sided bound \eqref{eq:target-global-one-sided} gives
\[
2\lambda\ip{e_k}{b(X_k)-b(Y_k)}
\le C\lambda\|e_k\|^2,
\]
while polynomial local Lipschitzness gives
\[
\lambda^2\|b(X_k)-b(Y_k)\|^2
\le
C\lambda^2 P(X_k,Y_k)\|e_k\|^2.
\]
The two-point taming consistency
\eqref{eq:taming-consistency-two-point} implies
\[
\lambda\|\tau_\lambda(X_k)-\tau_\lambda(Y_k)\|
\le C\lambda^3P(X_k,Y_k)\|e_k\|,
\]
and the resulting cross terms are absorbed into
\(C\lambda\|e_k\|^2+C\lambda^3P(X_k,Y_k)\) by Young's inequality and the
uniform moment bounds.  Consequently,
\[
\E\|R_k\|^2
\le
(1+C\lambda)u_k
+C\lambda^2\E[P(X_k,Y_k)\|e_k\|^2]
+C\lambda^3.
\]
The weighted term is controlled by H\"older and interpolation.  Fix a
high-moment order \(s>1\) and choose \(p>2\) so that, with
\(p'=p/(p-1)\), one has \(p'<2s/(s+1)\).  Interpolating between
\(L^2\) and \(L^{2s}\) then gives an exponent
\(\vartheta\in(1/2,1)\) such that
\[
\E[P(X_k,Y_k)\|e_k\|^2]
\le
\|P(X_k,Y_k)\|_{L^p}\|e_k\|_{L^{2p'}}^2
\le C u_k^\vartheta,
\]
where the uniform high-moment bounds absorb the \(L^{2s}\) factor.  Together with \eqref{eq:app-centered-square-compressed}
this yields
\[
u_{k+1}
\le
(1+C\lambda)u_k+C\lambda^2u_k^\vartheta+C\lambda^3.
\]
Setting \(v_k=u_k/\lambda^2\),
\[
v_{k+1}
\le
(1+C\lambda)v_k+C\lambda^{2\vartheta}v_k^\vartheta+C\lambda.
\]
Since \(\vartheta>1/2\), \(\lambda^{2\vartheta}\le\lambda\) for
\(\lambda\le1\), and \(v^\vartheta\le1+v\).  Discrete Gronwall on
\(k\lambda\le T\) yields \(\sup_{k\le n_\lambda}v_k\le C_T\) and hence
\(u_{n_\lambda}\le C_T\lambda^2\).

\subsection{Mean-oracle chain versus the target diffusion}
\label{app:mean-diffusion-details}

Let
\[
\rho_k
=
\int_{k\lambda}^{(k+1)\lambda}
\{b(Z_s)-b(Z_{k\lambda})\}\,\dd s.
\]
Applying It\^o's formula componentwise to \(b(Z_s)\) gives
\[
b(Z_s)-b(Z_{k\lambda})
=
\int_{k\lambda}^s\mathcal Lb(Z_u)\,\dd u
+
\sqrt{2\beta^{-1}}
\int_{k\lambda}^s\nabla b(Z_u)\,\dd B_u.
\]
The polynomial derivative bounds, finite-time diffusion moments, and
conditional It\^o isometry imply
\[
\|\E[\rho_k\mid\mathcal F_{k\lambda}]\|
\le C\lambda^2P(Z_{k\lambda}),
\qquad
\E[\|\rho_k-\E_k\rho_k\|^2\mid\mathcal F_{k\lambda}]
\le C\lambda^3P(Z_{k\lambda}).
\]
For the variance bound, the stochastic part is written by stochastic
Fubini as
\[
\sqrt{2\beta^{-1}}
\int_{k\lambda}^{(k+1)\lambda}
((k+1)\lambda-u)\nabla b(Z_u)\,\dd B_u,
\]
whose conditional second moment is bounded by
\(C\int_{k\lambda}^{(k+1)\lambda}((k+1)\lambda-u)^2
\E_kP(Z_u)\,\dd u\le C\lambda^3P(Z_{k\lambda})\).

Couple the mean-oracle chain and diffusion synchronously and write
\(e_k=Z_{k\lambda}-Y_k\).  With
\(m_k=\E[\rho_k\mid\mathcal F_{k\lambda}]\) and
\(\widehat\rho_k=\rho_k-m_k\),
\[
e_{k+1}
=
e_k+\lambda\{b(Z_{k\lambda})-b(Y_k)\}
-\lambda\tau_\lambda(Y_k)+m_k+\widehat\rho_k.
\]
The centered part has zero conditional cross term with the preceding
measurable summand.  Set
\(\delta_k=-\lambda\tau_\lambda(Y_k)+m_k\).  By
\eqref{eq:taming-consistency-point} and the local mean bound,
\(\|\delta_k\|\le C\lambda^2P(Z_{k\lambda},Y_k)\).  Hence
\[
2\ip{e_k}{\delta_k}
\le
\lambda\|e_k\|^2+C\lambda^{-1}\|\delta_k\|^2
\le
\lambda\|e_k\|^2+C\lambda^3P(Z_{k\lambda},Y_k),
\]
and the remaining cross term with
\(\lambda\{b(Z_{k\lambda})-b(Y_k)\}\) is bounded in the same way using
polynomial local Lipschitzness.  Combining these estimates with
\eqref{eq:target-global-one-sided} and
\eqref{eq:target-polynomial-local-lipschitz} gives
\[
w_{k+1}
\le
(1+C\lambda)w_k+C\lambda^2w_k^\vartheta+C\lambda^3,
\qquad
w_k=\E\|e_k\|^2,
\]
for some \(\vartheta\in(1/2,1)\).  The same scaled Gronwall argument as in
Appendix~\ref{app:block-accumulation-details} proves
\(w_{n_\lambda}\le C_T\lambda^2\) and hence
\eqref{eq:mean-block-w2-new}.

\section{Details for the polynomial moment bound}
\label{app:high-moment-details}

This appendix supplies the standard estimates used in the proof of
Theorem~\ref{thm:polynomial-foster-lyapunov}.  Throughout, \(m=2\ell\).

\subsection{Compact-region estimate}

Fix \(R<\infty\).  For \(t=\|\theta\|\le R\), write
\[
Y
=
\theta-\lambda H_{\lambda,R_0}(\theta,Z),
\]
where
\[
H_{\lambda,R_0}(\theta,Z)
=
\frac{\mathsf G(\theta,Z)+\eta\theta\|\theta\|^{2r}}
{D_{\lambda,R_0}(\|\theta\|)}.
\]
Since \(D_{\lambda,R_0}\ge1\) and
\(\|\mathsf G(\theta,Z)\|\le\mathcal K(Z)(1+\|\theta\|^q)\),
\[
\sup_{\|\theta\|\le R}
\E\|H_{\lambda,R_0}(\theta,Z)\|^m
\le
K_{m,R}
\]
uniformly for \(0<\lambda\le1\).  The polynomial increment inequality
\[
\|u+v\|^m
\le
\|u\|^m
+
C_m'
\sum_{j=1}^m
\|u\|^{m-j}\|v\|^j
\]
therefore gives
\[
\E_U\|Y\|^m
\le
t^m+B_{m,R}\lambda,
\qquad
t\le R,
\]
because \(\lambda^j\le\lambda\) for \(1\le j\le m\) and
\(0<\lambda\le1\).  After increasing \(B_{m,R}\),
\[
\E_U\|Y\|^m
\le
(1-c\lambda)t^m+B_{m,R}\lambda
\]
holds on the same ball for any fixed \(c>0\).

\subsection{Expansion of the stochastic-gradient remainder}

For \(t\ge R\), recall that
\[
Y
=
\bigl(1-\eta a_\lambda(t)\bigr)\theta+E
\]
and
\[
\|E\|
\le
a_\lambda(t)\mathcal K(Z)\delta(t)t.
\]
Applying the polynomial increment inequality directly to these two terms,
without introducing additional notation, gives
\[
\begin{aligned}
\|Y\|^m
&\le
|1-\eta a_\lambda(t)|^m t^m
+
C_m'
\sum_{j=1}^m
|1-\eta a_\lambda(t)|^{m-j}
a_\lambda(t)^j
\mathcal K(Z)^j
\delta(t)^j
t^m.
\end{aligned}
\]
Let \(M_A=\max_{0\le a\le A}|1-\eta a|.\) After enlarging \(R\), we may assume \(\delta(t)\le1\).  Since
\(a_\lambda(t)\le A\),
\[
a_\lambda(t)^j\delta(t)^j
\le
A^{j-1}a_\lambda(t)\delta(t).
\]
Taking expectation in \(Z\), define
\[
C_m
=
C_m'
\sum_{j=1}^m
M_A^{m-j}A^{j-1}\E[\mathcal K(Z)^j].
\]
The assumption \(\E\mathcal K(Z)^m<\infty\) ensures \(C_m<\infty\).
Consequently,
\[
\E_U\|Y\|^m
\le
|1-\eta a_\lambda(t)|^m t^m
+
C_m a_\lambda(t)\delta(t)t^m,
\]
which is \eqref{eq:pre-noise-tail-main}.

\subsection{Gaussian increment}

Let \(\sigma=\sqrt{2\beta^{-1}}\) and \(\xi\sim N(0,I_d)\).  For
\(f(z)=\|z\|^m\),
\[
\|\nabla^2 f(z)\|_{\mathrm{op}}
\le
m(m-1)\|z\|^{m-2}.
\]
Applying the second-order Taylor formula to
\(f(y+\sigma\sqrt\lambda\,\xi)\), the first-order term vanishes after
expectation because \(\E\xi=0\).  Hence, for a constant
\(C_m^{\mathrm G}<\infty\),
\[
\E_\xi
\left[
\|y+\sigma\sqrt\lambda\,\xi\|^m
\right]
\le
\|y\|^m
+
C_m^{\mathrm G}
\left(
\lambda\|y\|^{m-2}
+
\lambda^{m/2}
\right).
\]
Since \(m\ge2\) and \(0<\lambda\le1\),
\[
\lambda^{m/2}\le\lambda,
\]
and therefore
\begin{equation}\label{eq:gaussian-moment-app}
\E_\xi
\left[
\|y+\sigma\sqrt\lambda\,\xi\|^m
\right]
\le
\|y\|^m
+
C_m^{\mathrm G}\lambda
\bigl(1+\|y\|^{m-2}\bigr).
\end{equation}

For every \(\varepsilon>0\), Young's inequality gives
\[
\|y\|^{m-2}
\le
\varepsilon\|y\|^m+C_{m,\varepsilon}.
\]
Applying \eqref{eq:gaussian-moment-app} conditionally on \(Y\), taking
expectation in \(Z\), and using \eqref{eq:pre-noise-global-main}, we obtain
\[
\begin{aligned}
R_\lambda V_m(\theta)
&\le
1+
\bigl(1+C_m^{\mathrm G}\varepsilon\lambda\bigr)
\bigl[(1-c_m\lambda)t^m+B_m\lambda\bigr]
+
C_{m,\varepsilon}'\lambda.
\end{aligned}
\]
Choose \(\varepsilon>0\) so that
\(C_m^{\mathrm G}\varepsilon<c_m/2\), and then decrease \(\lambda_0\) if
necessary.  The preceding display becomes
\[
R_\lambda V_m(\theta)
\le
(1-a_m\lambda)V_m(\theta)+b_m\lambda
\]
for suitable \(a_m>0\) and \(b_m<\infty\), proving
\eqref{eq:V2p-main}.

\end{document}